\documentclass[12pt]{article}
\usepackage{graphicx} 
\usepackage{amsmath}  
\usepackage{amsfonts} 
\usepackage{amssymb}  
\usepackage{hyperref} 
\usepackage{geometry} 
\usepackage{theme}
\usepackage{mathtools}
\usepackage[numbers]{natbib}
\usepackage{shortcuts}
\usepackage{algorithm}
\usepackage{algpseudocode}
\usepackage[utf8]{inputenc}
\usepackage[T1]{fontenc}
\usepackage{amsmath}
\usepackage{multirow}
\usepackage{amssymb}
\setlist[enumerate]{itemsep=0mm}
\usepackage{amsfonts}
\usepackage{graphicx}
\usepackage{hyperref}
\usepackage{color}
\usepackage{ulem}
\usepackage{diagbox}
\usepackage{array}
\usepackage{tabularx}
\usepackage{float}

\newcolumntype{C}[1]{>{\centering\arraybackslash}m{#1}}
\newcolumntype{Y}{>{\centering\arraybackslash}X}

\newcommand{\HS}{\mathcal{H}}
\renewcommand{\ker}{k}

\newcommand{\bsx}{\mathbf{x}}

\newcommand{\MFker}{\textit{\mbox{MF ker-only}}}
\newcommand{\MFmuker}{\textit{\mbox{MF mean-ker}}}
\newcommand{\MFmu}{\textit{\mbox{MF mean-only}}}
\newcommand{\dist}{\mathcal{D}}
\newcommand{\Nobs}{N_{\text{MC}}}

\title{Multifidelity Gaussian process regression for solving nonlinear partial differential equations }
\author{
    Fatima-Zahrae El-Boukkouri\thanks{Institut de Mathématiques de Toulouse, Universit\'e de Toulouse, INSA, Toulouse, 31077, France, \texttt{el-boukkouri@insa-toulouse.fr}} 
    \and
    Josselin Garnier\thanks{CMAP, CNRS, Ecole polytechnique, Institut Polytechnique de Paris, Palaiseau, 91120, France, \texttt{josselin.garnier@polytechnique.edu}} \and
    Olivier Roustant\thanks{Institut de Mathématiques de Toulouse, Universit\'e de Toulouse, INSA, Toulouse, 31077, France, \texttt{roustant@insa-toulouse.fr}} 
}
\begin{document}
\maketitle

\begin{abstract}
Solving nonlinear partial differential equations (PDEs) using kernel methods offers a compelling alternative to traditional numerical solvers. However, the performance of these methods strongly depends on the choice of kernel. In this work, as the available information is inherently multifidelity, we propose a kernel learning approach based on cokriging, leveraging empirical information from multifidelity simulations. In the first step, we fit a differentiable non-stationary kernel to an empirical kernel obtained from low-fidelity simulations. In the second step, we derive a high-fidelity kernel with estimated hyperparameters, and construct a corresponding high-fidelity mean using the multifidelity framework. These components can then be used within a Gaussian process framework for solving PDEs. Finally, we demonstrate the performance of the proposed physics-informed method on the Burgers' equation.
\end{abstract}

\section{Introduction}

\subsection*{Context and aim of the paper}

Partial differential equations (PDEs) constitute the standard mathematical framework for modeling conservation laws and transport phenomena in physics. Their theoretical analysis and numerical approximation have been extensively developed~\cite{evans2022partial,grossmann2007numerical,brenner2008mathematical,quarteroni1994numerical}. Classical solvers such as finite difference and finite element methods enforce the governing equations at the discrete level and rely on well-established stability analyses. However, their computational cost can become significant in nonlinear or high-dimensional regimes, while purely data-driven approaches may struggle to preserve structural physical constraints when data are limited~\cite{raissi2019physics,karniadakis2021physics}. These challenges motivate hybrid methodologies combining statistical learning with explicit physical modeling.

In many practical applications, the available information is inherently multifidelity. In particular, low-fidelity simulators provide inexpensive but approximate solutions, while high-fidelity measurements or simulations are accurate but can be sparse. The aim is to reconstruct the solution of the governing PDE by consistently combining these heterogeneous sources of information while respecting the underlying physical constraints. A typical example arises in coastal flooding prediction, where fast approximate models are combined with sparse observations to reconstruct the dynamics of marine submersion. Gaussian processes (GPs) provide a natural probabilistic framework for this task~\cite{williams2006gaussian,stein1999interpolation,cressie1993statistics}. In Gaussian process regression (GPR), the unknown field is modeled as a GP characterized by a mean and a covariance function, which encodes prior assumptions on regularity and correlation structure. Its multifidelity extension, co-kriging~\cite{kennedy, Cokriging}, combines correlated information sources and has become a standard tool for hierarchical modeling of computer experiments~\cite{forrester2007multi,legratiet2013recursive,peherstorfer2018survey}. More recently, physics-informed Gaussian process methodologies have incorporated partial physical knowledge directly into the prior construction or covariance design~\cite{yang2018physics,Owhadi}. This leads to probabilistic numerical formulations that combine Bayesian inference with PDE constraints~\cite{stuart2010inverse,cockayne2019bayesian}.

\subsection*{Related works}

Recent works have explored the use of GPs for solving nonlinear PDEs in a physics-informed setting. In particular, the framework proposed by Chen et al.~\cite{Owhadi} formulates PDE resolution as a GPR problem in a reproducing kernel Hilbert space (RKHS), building upon classical kernel-based approximation and meshfree methods~\cite{wendland2004scattered,fasshauer2007meshfree}. In this formulation, the solution is obtained as the minimum-norm element satisfying the PDE constraints at collocation points, providing a rigorous variational interpretation together with convergence guarantees. However, the performance of this approach strongly depends on the choice of the covariance kernel, which determines the structure of the admissible solution space. To address this sensitivity, several works have proposed to learn covariance structures in a physics-informed manner, including physics-informed kernel learning~\cite{doumeche2025physics} and physics-informed deep kernel learning~\cite{wang2022physics}, which combine data and physical constraints to improve prediction and uncertainty quantification. 
Nelsen et al.~\cite{owhadi_parameters} proposed a kernel-learning strategy to construct covariance functions adapted to the underlying PDE using a bilevel optimization framework. Their approach relies on two sets of data: one used to solve the PDE through GPR, and another used to learn the kernel parameters. The latter are optimized through a bilevel learning procedure, allowing the kernel to adapt to the underlying PDE solution.

However, these approaches are mainly developed for single-fidelity settings, which assume all data comes from a single, usually high-fidelity, source. As a result, they do not exploit the hierarchical structure of multifidelity data present in practical applications. By contrast, this work constructs kernels in a multifidelity framework: it uses inexpensive low-fidelity simulations to learn the main covariance structure and then refines it with high-fidelity information.
\subsection*{Contributions}
In contrast to approaches that focus on tuning parametric kernel families, our work addresses the construction of the covariance kernel itself within a multifidelity, physics-informed framework. First, we build upon the multifidelity modeling framework introduced by Kennedy and O’Hagan~\cite{kennedy} and its co-kriging formulation~\cite{Cokriging}, where information from low- and high-fidelity sources is fused through correlated GPs. In this setting, the covariance structure of the low-fidelity process is typically estimated empirically from simulator output values. However, this empirical covariance is defined only on the simulation grid and is generally not smooth enough to be used in PDE-based Gaussian process solvers~\cite{Owhadi}. To address this issue, we introduce a regularization and smoothing procedure that approximates the empirical covariance by a differentiable kernel belonging to a low-dimensional parametric family. This construction produces a smooth kernel that preserves the dominant variability observed in the low-fidelity simulations while remaining compatible with the regularity requirements imposed by the governing PDE. Compared with recent approaches that learn kernel parameters through large-scale bilevel optimization~\cite{owhadi_parameters}, the proposed procedure is computationally inexpensive and involves only a small number of parameters. Finally, we use the resulting kernel within the GP PDE-solving framework of \cite{Owhadi}.

In addition to this kernel construction, we investigate the role of the GP mean in a multifidelity setting and propose a mean-correction strategy incorporating high-fidelity information. These constructions lead to three variants of the method: the first one incorporates multifidelity information through the kernel only; the second one incorporates multifidelity information through both the kernel and the mean; and the third one uses multifidelity information through the mean only, while learning the kernel directly from high-fidelity residuals.

Furthermore, we establish that, when the mean coincides with the true solution, the method exactly recovers the PDE solution. Finally, we extend this formulation to incorporate simultaneously PDE constraints and high-fidelity sensor measurements within the RKHS variational problem. This unified formulation allows us to combine physical constraints, multifidelity simulations, and sparse observations in a single GP reconstruction framework.

\subsection*{Outline of the paper}
The rest of the paper is organized as follows. 
In Section~\ref{sec:background}, we briefly recall the GP and RKHS framework used for PDE solving, together with the multifidelity cokriging model of Kennedy and O’Hagan. We then present, in Section~\ref{Sec:PhI_methodology}, 
the physics-informed multifidelity methodology, introduce three kernel-learning strategies and discuss the classes of kernels considered. 
In Section~\ref{sec:linburgers}, we illustrate, first, the approach on a linearized Burgers' equation, which allows us to compare the three methodologies and provide theoretical insight into the mean-corrected formulations. 
The nonlinear Burgers' equation is then studied in Section~\ref{sec:burgers}, where we highlight the difficulty of kernel selection in a single-fidelity setting and demonstrate the improvements obtained with the proposed multifidelity strategy.
Some conclusions and perspectives are finally drawn in Section~\ref{sec:conclusion}.

\section{Background and notation}
\label{sec:background}
This section provides the two building blocks used in the remainder of the paper:
(i) the PDE--RKHS variational formulation of~\cite{Owhadi}, which turns PDE solving into a minimum-norm problem in a kernel space, and (ii) the multifidelity cokriging model of~\cite{kennedy,Cokriging}, which enables us to learn covariance structure from inexpensive simulations and transfer it to the high-fidelity regime. In Section~\ref{Sec:PhI_methodology}, we combine these ingredients to construct PDE-adapted RKHS priors in a multifidelity setting.

\subsection{Solving non linear PDEs with Gaussian processes}
We consider a bounded domain \(\Omega \subseteq \mathbb{R}^{p}\) for \(p \geq 1\) and a nonlinear PDE of the form:
\[ 
\begin{cases}
\label{eq:PDE}
\mathcal{P}\left(u^{\star}\right)(\bsx) = f(\bsx), & \forall \bsx \in \Omega, 
\\
\mathcal{B}\left(u^{\star}\right)(\bsx) = g(\bsx), & \forall \bsx \in \partial \Omega,
\end{cases}
\]

where \(\mathcal{P}\) is a nonlinear differential operator, \(\mathcal{B}\) is a boundary operator, and \(f, g\) are given functions. We assume that the PDE is well-defined pointwise and has a unique strong solution. Let \(\mathcal{H}(k)\) be a suitable RKHS for the solution \(u^{\star}\), with an associated kernel $\ker$. 

Chen et al. \cite{Owhadi} propose approximating \(u^{\star}\) using a GP with covariance kernel $\ker$
conditioned to satisfy the PDE at a finite set of collocation points in \(\overline{\Omega}\), and then computing the Maximum A Posteriori (MAP) of this conditioned GP. Let \(\left\{\bsx_{i}\right\}_{i=1}^{M}\) be a set of points in \(\overline{\Omega}\), with \(\bsx_{1}, \ldots, \bsx_{M_{\Omega}} \in \Omega\) (interior) and \(\bsx_{M_{\Omega}+1}, \ldots, \bsx_{M } \in \partial \Omega\) (boundary).
The optimization problem to solve is then:
\begin{equation}
\underset{u \in \HS(\ker)}{\operatorname{minimize}} \, \|u\|_{\HS(\ker)} \quad
\text{s.t.} \quad 
\begin{cases}
\mathcal{P}(u)(\bsx_{m}) = f(\bsx_{m}), & m = 1,\ldots, M_{\Omega},\\[2mm]
\mathcal{B}(u)(\bsx_{m}) = g(\bsx_{m}), & m = M_{\Omega}+1,\ldots, M.
\end{cases}
\label{relaxation1}
\end{equation}

This means that we aim to approximate \(u^{\star}\) with the minimum norm element in $\HS(k)$ that satisfies the PDE and boundary conditions at the collocation points \(\left\{\bsx_{i}\right\}_{i=1}^{M}\).
The methodology for solving \eqref{relaxation1} and its theoretical guarantees are detailed in \cite{Owhadi}.
 In this paper, we adopt this approach and focus specifically on the choice of a suitable RKHS $\HS(k)$, which plays a crucial role in capturing the properties of the solution $u^{\star}$.

\subsection{The cokriging model}
In this subsection, we present the two-level multifidelity modeling approach proposed by \cite{kennedy}, which is based on a co-kriging model.
The prior on the low-fidelity function is modeled as a GP:
\[
Y_L \sim \mathcal{GP}(\mu_L, \ker_L),
\]
and the prior of the high-fidelity function is modeled similarly as:
\[
Y_H \sim \mathcal{GP}(\mu_H, \ker_H)  .
\]

\cite{kennedy} proposed a multifidelity formulation based on the autoregressive model for \( Y_H \):
\begin{equation}
Y_H = \rho Y_L + Y_d,
\end{equation}
where \( \rho \in \mathbb{R} \) is a regression parameter, and 
\( Y_d \sim \mathcal{GP}(\mu_d, \ker_d) \) 
is a GP that models the residual term
\( Y_H - \rho Y_L \).
The AR(1) structure follows from a Markov property, that is to say, 
the conditional independence between \( Y_H(\bsx) \) and \( Y_L(\bsx') \), given \( Y_L(\bsx) \), for all \( \bsx' \neq \bsx \) in the domain \( D \), i.e.,
\begin{equation}
\text{Cov}\left\{ Y_H(\bsx), Y_L(\bsx') \mid Y_L(\bsx) \right\} = 0.
\end{equation}
We consider both high-resolution data \( y_H \) and low-resolution data \( y_L \), observed respectively on the grids \( X_H \) and \( X_L \), with \( X_H \subseteq X_L \). This hierarchical structure enables 
efficient and rapid hyperparameter selection by maximum likelihood.

\section{Physics-informed cokriging for solving PDEs}
\label{Sec:PhI_methodology}

\subsection{Our Methodology}
\label{subsec:nonstationary-kernel}
Our objective is to construct a RKHS tailored to the numerical solution of PDEs. We assume access to:
\begin{itemize}
    \item A stochastic low-fidelity simulator providing $\Nobs$ realizations 
    $\left( \boldsymbol{y}_L^{\,i} \right)_{i=1}^{\Nobs}$ 
    on a grid $X_L= \{\bsx_{L,i}\}_{i=1}^{N_L} $, where $\bsx_{L,i} \in \mathbb{R}^p$ and $N_L$ denotes the number of points in the low-fidelity grid. 
    
    \item A high-fidelity simulator or experiment providing observations 
    $\boldsymbol{y}_H$ on a grid 
    $X_H \subseteq X_L$, where
    \[
    \boldsymbol{y}_H = (y_{H,1}, \ldots, y_{H,N_H}) \in \mathbb{R}^{N_H},
    \quad
    X_H = \{\bsx_{H,i}\}_{i=1}^{N_H} , 
    \]
    where $\bsx_{H,i} \in \mathbb{R}^p$ and $N_H$ denotes the number of points in the high-fidelity grid.
\end{itemize}

Our strategy is to place a GP prior with mean $\mu_H$ and covariance kernel $\ker_H$
on the high-fidelity solution and to define the numerical approximation as the MAP estimator under the PDE constraints and the high-fidelity observations.
The key point is that the prior mean \(\mu_H\) and covariance kernel \(\ker_H\) are constructed using the multifidelity framework described in Section~\ref{sec:background}. Depending on how this multifidelity information is incorporated, different choices of \((\mu_H, \ker_H)\) lead to different RKHS structures and therefore to different numerical methods.
The three constructions considered in this work are summarized in Table~\ref{tab:methodologies}, and detailed in the next paragraphs.

\begin{table}[H]
\centering
\small
\begin{tabularx}{0.96\textwidth}{|C{2.6cm}||Y|Y||Y|}
\hline
\diagbox[innerwidth=\linewidth]{$\mu_H$}{$\ker_H$}
& $\ker'_H$
& $\rho^2 \ker_{\mathrm{opt}} + \ker_d$
& Comments on $\mu_H$ \\
\hline \hline

$0$
& Not a multifidelity setting
& {\MFker}
& \\
\hline

$\hat{\mu}_H$
& {\MFmu}
& {\MFmuker}
& $\hat{\mu}_H:=$ GP prediction of $\rho\mu_L+\mu_d$ given observations. \\
\hline \hline

Comments on $\ker_H$
& $\ker'_H:=$ selected in a given family of kernels.
& $\ker_{\mathrm{opt}}:=$ smooth kernel, proxy of 
$\ker_L$. \newline
  $\ker_d:=$ anisotropic Gaussian kernel.
& \\
\hline
\end{tabularx}
\caption{Summary of the three methodologies \MFker, \MFmu, \MFmuker, according to the choice of mean and kernel.}
\label{tab:methodologies}
\end{table}

\subsubsection*{{\MFker}: Multifidelity kernel learning}
We consider here the following problem  
\begin{equation}
\underset{u \in \HS(\ker)}{\operatorname{minimize}} \, \|u\|_{\HS(\ker)} \quad
\text{s.t.} \quad 
\begin{cases}
\mathcal{P}(u)(\bsx_{m}) = f(\bsx_{m}), & m = 1,\ldots, M_{\Omega},\\[2mm]
\mathcal{B}(u)(\bsx_{m}) = g(\bsx_{m}), & m = M_{\Omega}+1,\ldots, M, \\[2mm]
u(\bsx_{H, i}) = y_{H,i},  & i = 1, \ldots, N_H ,
\end{cases}
\label{eq:Methodology1}
\end{equation}
where the kernel $\ker$ is unknown, and must be estimated.
In this methodology, we construct the kernel $\ker$ using a $3$-stage methodology: estimation of the empirical kernel of the low-fidelity model, approximation by a smooth kernel, resolution of the problem.

\paragraph{Step 1: empirical kernel estimation for the LF simulator.} The cokriging model described in \cite{Cokriging} is based on $\Nobs$ realizations $(\boldsymbol{y}_L^i)_{i=1}^{\Nobs}$ of the low-fidelity simulator on the grid $X_L$ and high-fidelity observations $\boldsymbol{y}_H$ on $X_H$.
For that, we begin by estimating the empirical mean and covariance of the low-fidelity GP:
\begin{equation}
    \mu_L(\bsx) = \mu_{\text{MC}}(\bsx) = \frac{1}{\Nobs} \sum_{i=1}^{\Nobs} \boldsymbol{y}_L^i(\bsx),
    \label{eq:MCmean}
\end{equation}
\begin{equation}
    \ker_L(\bsx,\bsx') = \ker_{\text{MC}}(\bsx,\bsx') = \frac{1}{\Nobs-1} \sum_{i=1}^{\Nobs} \left( \boldsymbol{y}_L^i(\bsx) - \mu_{\text{MC}}(\bsx) \right) \left( \boldsymbol{y}_L^i(\bsx') - \mu_{\text{MC}}(\bsx') \right),
    \label{eq:MCkernel}
\end{equation}
for $\bsx,\bsx'\in X_L$.
The high-fidelity discrepancy is modeled as a stationary GP with constant mean $\mu_d \in \mathbb{R}$:
\[
Y_d \sim \mathcal{GP}(\mu_d, \ker_d),
\quad
\ker_d(\bsx, \bsx')=\sigma_d^2
\exp\!\left(- \sum_{s=1}^p\frac{(x_s-x_s')^2}{2 \ell_s^2}\right).
\]
We consider the GP:
\[
Y_H = \rho Y_L + Y_d,
\]
leading to the log marginal likelihood
\begin{multline}
\ln L  = -\frac{1}{2} ( \boldsymbol{y}_H - \rho \mu_L(X_H) - \mu_d)^\top C_H^{-1} (\boldsymbol{y}_H - \rho \mu_L(X_H) - \mu_d) \\
  - \frac{1}{2} \ln |C_H| - \frac{N_H}{2} \ln 2\pi,
\label{eq:logL}
\end{multline}
where $C_H$ is the covariance matrix 
\[
(C_H)_{ij} = \rho^2 \ker_L( \bsx_{H,i}, \bsx_{H,j}) + \ker_d( \bsx_{H,i}, \bsx_{H,j}), \quad 1 \leq i, j \leq N_H.
\]

\begin{remark}
In \cite{Cokriging}, the regression parameter $\rho$ and the hyperparameters of the discrepancy process $Y_d = Y_H - \rho Y_L$ are estimated by maximizing the likelihood of $Y_d$. This requires observing a realization of $Y_d$ on $X_H$, and therefore access to a realization of $Y_L$ on $X_H$.
However, in our setting, the low-fidelity simulator is stochastic and provides multiple realizations. As a result, $Y_L$ is not deterministic on $X_H$, but is instead observed through a collection of samples. In the co-kriging framework of \cite{Cokriging}, this issue is typically addressed by considering the empirical mean $\mu_L(X_H)$ as a proxy for $Y_L(X_H)$.
In contrast, we adopt a different strategy to account for this stochasticity. Instead of working with a residual formulation, we consider the Gaussian process model
\[
Y_H = \rho Y_L + Y_d,
\]
which induces a Gaussian process with mean $\rho \mu_L + \mu_d$ and covariance $\rho^2 \ker_L + \ker_d$. The hyperparameters are then estimated by maximizing the corresponding marginal likelihood.
\end{remark}

\paragraph{Step 2. Smoothing of the LF kernel.}
Since \(\ker_L\) is empirical and defined only on \(X_L\), we replace it by a smooth kernel \(\ker_{\mathrm{opt}}\) obtained by solving
\[
\ker_{\mathrm{opt}}=\arg\min_{\ker\in S} 
\dist \bigl(\ker|_{X_L\times X_L},\, \ker_L\bigr),
\]
where \(S\) is a class of infinitely differentiable kernels and \( \dist \) is a matrix distance.

\paragraph{Step 3. Construction of a HF kernel and resolution of the PDE.}
The final multifidelity kernel is
\[
\ker_H^* = \rho^2 \ker_{\mathrm{opt}} + \ker_d.
\]

The RKHS \( \mathcal{H}(\ker_H^*) \) is then used to solve the PDE.

\begin{algorithm}[H]
\caption{{\MFker}}
\begin{algorithmic}[1]
\Require Low-fidelity realizations \( \boldsymbol{y}_L^i \), high-fidelity data \( \boldsymbol{y}_H \)
\Ensure Solution $u$ of the PDE 
\State Compute the empirical mean \( \mu_L \) and the empirical covariance \( \ker_L \) using~\eqref{eq:MCmean}--\eqref{eq:MCkernel}
\State Estimate \( \rho, \mu_d, \sigma_d, (\ell_s)_{s=1, \dots, p} \) by maximizing the log likelihood ~\eqref{eq:logL} 
\State Select a smooth kernel class \(S\) and a distance metric \(\dist\)
\State Compute \( \ker_{\mathrm{opt}} = \arg\min_{\ker\in S} \dist(\ker,\ker_L) \)
\State Set \(\ker =  \ker_H^* = \rho^2 \ker_{\mathrm{opt}} + \ker_d \)
\State Solve the RKHS minimization problem  \eqref{eq:Methodology1} in \( \mathcal{H}(\ker) \)
\State Return $u$
\end{algorithmic}
\end{algorithm}

\subsubsection*{{\MFmuker}: High-fidelity mean correction with multifidelity kernel learning}
In this methodology, we incorporate the effect of the mean $\mu_H = \rho \mu_L + \mu_d $ constructed by cokriging. However, since $\mu_H$ is defined only on the grid $X_L$, we construct a smooth extension $\widehat{\mu}_H$ using GPR.
Thus, we consider the model:
\[
y_H(\bsx) = \mu_H(\bsx) + \varepsilon(\bsx),
\qquad
\varepsilon \sim \mathcal{GP}(0, \ker_H^*).
\]
To estimate $\mu_H$, we perform GPR with an anisotropic Gaussian kernel
\[
\ker_\theta(\bsx, \bsx')
=
\sigma^2 \exp\!\left(
-\sum_{s=1}^p \frac{(x_s - x_s')^2}{2\ell_s^2}
\right),
\]
and a regularization parameter $\lambda$ (also known as ``nugget'').
The resulting estimator is
\[
\widehat{\mu}_H(\bsx)
=
\ker_\theta(\bsx, X_H)
\bigl(\ker_\theta(X_H, X_H) + \lambda I\bigr)^{-1}
\mu_H(X_H).
\]

The hyperparameters $(\theta, \lambda)$, with $\theta=(\sigma^2,l_s,s=1,\ldots,p)$, are obtained by maximizing the log-likelihood
\begin{equation} \label{eq:logLikGPRforMu}
\log L(\theta,\lambda)
=
-\frac12
\mu_H(X_H)^\top
\bigl(K_{\theta, \lambda} \bigr)^{-1}
\mu_H(X_H)
-\frac12 \log\!\bigl|\ker_{\theta, \lambda} \bigr|
-\frac{N_H}{2}\log(2\pi),
\end{equation}
with $ K_{\theta, \lambda} = \ker_\theta(X_H,X_H)+\lambda I $.

Finally, the model is:
\[
y_H(\bsx) = \widehat{\mu}_H(\bsx)
 + \varepsilon(\bsx),
\qquad
\varepsilon \sim \mathcal{GP}(0, \ker_H^*).
\]

The PDE is then solved in $\mathcal{H}(\ker_H^*)$ through the optimization problem:
\begin{equation}
\displaystyle \min_{h \in \mathcal{H}(\ker_H^*)} \|h\|_{\mathcal{H}(\ker_H^*)} \quad
\text{s.t.} \quad 
\begin{cases}
\mathcal{P}(h + \widehat{\mu}_H )(\bsx_{m}) = f(\bsx_{m}), & m = 1,\ldots, M_{\Omega},\\[2mm]
\mathcal{B}(h+ \widehat{\mu}_H) (\bsx_{m}) = g(\bsx_{m}), & m = M_{\Omega}+1,\ldots, M, \\[2mm]
h(\bsx_{H, i}) + \widehat{\mu}_H(\bsx_{H, i})= y_{H,i},  & i = 1, \ldots, N_H , 
\end{cases}
\label{eq:methodology23}
\end{equation}
and the final solution is $u = h + \widehat{\mu}_H$.

In {\MFmuker}, we follow the same kernel construction procedure as in {\MFker}, 
leading to the multifidelity kernel $\ker = \ker_H^{*}$. 
The difference lies in the final step: instead of solving problem~\eqref{eq:Methodology1}, 
we solve problem~\eqref{eq:methodology23}.

\begin{algorithm}[H]
\caption{{\MFmuker}}
\begin{algorithmic}[1]
\Require The vector $\mu_H:= \rho \mu_L + \mu_d$ and the function \( \ker_H^* \) from {\MFker}
\Ensure Corrected solution \( u = h + \widehat{\mu}_H \)
\State Estimate $\theta, \lambda$ by maximizing the log-likelihood \eqref{eq:logLikGPRforMu}
\State Estimate \( \widehat{\mu}_H \) by GPR on $\mu_H$
with kernel \(\ker_\theta\)
\State Form shifted operators:  
\(\mathcal{P}'(h)=\mathcal{P}(h+\widehat{\mu}_H)\),  
\(\mathcal{B}'(h)=\mathcal{B}(h+\widehat{\mu}_H)\)
\State Set $\ker = \ker_H^* $
\State Solve the minimization problem \eqref{eq:methodology23} in the RKHS \( \mathcal{H}(\ker) \)
\State Return \( u = h + \widehat{\mu}_H \)
\end{algorithmic}
\end{algorithm}

We note that {\MFker}  solves the PDE directly for \(u\), whereas {\MFmuker} solves for the residual \(u - \widehat{\mu}_H\) using shifted operators. Both approaches rely on the same RKHS with kernel
\[
\ker_H^{*} = \rho^2 \ker_{\mathrm{opt}} + \ker_d,
\]
and their main difficulty lies in constructing the smooth kernel \(\ker_{\mathrm{opt}}\) that approximates the empirical covariance \(\ker_L\) on \(X_L\). To avoid this approximation step, we propose a third methodology in which the kernel of the RKHS is learned directly from the high-fidelity data. In this setting, the multifidelity information is carried by the mean function only, while the covariance kernel is inferred from the high-fidelity residuals. This approach is effective when the number of high-fidelity observations is sufficiently large, unlike the cokriging-based methods ({\MFker} and {\MFmuker}), which remain applicable even with limited high-fidelity information.
\subsubsection*{{\MFmu}: Mean correction with direct high-fidelity kernel learning }

In the third methodology, we solve problem~\eqref{eq:methodology23}, but we abandon the multifidelity construction of the kernel $\ker$ and instead learn a smooth kernel directly from the high-fidelity residuals.

We consider
\[
r(X_H) = y_H(X_H) - \mu_H(X_H),
\qquad
r \sim \mathcal{GP}(0, \ker_H').
\]

We then construct the kernel $\ker_H'$ by choosing a class  of positive-definite kernels \( S' \) and maximizing the following log-likelihood: 
\[
\log L(\ker)
=
-\frac12\, r(X_H)^\top \ker(X_H,X_H)^{-1}\, r(X_H)
-\frac12 \log\!\bigl|\ker(X_H,X_H)\bigr|
-\frac{N_H}{2}\log(2\pi).
\]

The PDE is then solved in \( \mathcal{H}(\ker_H') \) using the same shifted-operator formulation as in {\MFmuker}.

\begin{algorithm}[H] 
\caption{ {\MFmu}}
\label{algo:MFmu}
\begin{algorithmic}[1]
\Require HF data \(\boldsymbol{y}_H\), HF mean \( \widehat{\mu}_H \) 
\Ensure Solution $u$ of the PDE
\State Compute residuals \( r(X_H) = y_H(X_H)-\mu_H(X_H)\) 
\State Select a smooth kernel class \( S' \) 
\State Maximum likelihood estimation: \( \ker_H' = \arg\max_{\ker\in S'} \log L(\ker) \) 
\State Set $\ker = \ker_H'$
\State Solve the minimization problem \eqref{eq:methodology23} in the RKHS \( \mathcal{H}(\ker) \)
\State Return \( u = h + \widehat{\mu}_H \)
\end{algorithmic} 
\end{algorithm}

\subsection{Tuning details on the smooth kernel construction of Step 2}
A simple and widely used choice for the class \( S \) is the family of anisotropic stationary kernels.  
Instead of restricting ourselves to the squared–exponential (Gaussian) kernel, we consider a broader class
that includes Gaussian, Matérn--\(3/2\), Matérn--\(5/2\), and Matérn--\(7/2\) kernels.  
This family may be written as
\[
S = \left\{
\ker \in \mathcal{F}(\mathbb{R}^p,\mathbb{R}):
\ker(\bsx, \bsx') = \sigma^2\, R^S\!\left(\|\bsx - \bsx'\|_{\theta}\right),
\ \sigma>0,\ \theta_s>0
\right\},
\]

where  $\|\bsx- \bsx'\|_{\theta} = \sqrt{\sum_{s=1}^{p}\frac{(x_s-x'_s)^2}{\theta_s^2}}$ and \(R^S: [0,\infty) \to \mathbb{R}\) is any stationary isotropic correlation function.  
Typical choices for \(R^S\) include:
\begin{align*}
\text{Gaussian:}\quad
& R^S(r) = \exp\!\left(-\frac{r^2}{2}\right), \\[1mm]
\text{Matérn--3/2:}\quad
& R^S(r) = \left(1 + \sqrt{3}\, r \right)\exp(-\sqrt{3}\, r), \\[1mm]
\text{Matérn--5/2:}\quad
& R^S(r) = \left(1 + \sqrt{5}\, r + \tfrac{5}{3}r^2\right)\exp(-\sqrt{5}\, r), \\[1mm]
\text{Matérn--7/2:}\quad
& R^S(r) = \left(1 + \sqrt{7}\, r + \tfrac{14}{5} r^2 + \tfrac{7}{15} r^3\right)\exp(-\sqrt{7}\, r).
\end{align*}

However, the stationarity of these kernels limits their expressiveness when the underlying phenomenon 
exhibits location-dependent behaviour.  
To address this limitation, we also consider a nonstationary extension of a generic stationary kernel 
\(R^S\), following the construction in \cite{NSkernel}.  
The resulting nonstationary correlation function is
\begin{equation}
R_{\Sigma}^{NS}(\bsx, \bsx') =
|\Sigma_\bsx|^{1/4} |\Sigma_{\bsx'}|^{1/4}
\left|\frac{\Sigma_\bsx + \Sigma_{\bsx'}}{2}\right|^{-1/2}
R^S\!\left(\sqrt{(\bsx - \bsx')^\top 
\left(\tfrac{\Sigma_\bsx + \Sigma_{\bsx'}}{2}\right)^{-1}
(\bsx - \bsx')}\right),
\label{eq:NS_gau_ke}
\end{equation}
where \(\Sigma_\bsx \) is a local covariance (or length-scale) matrix at input \( \bsx\).

Although expressive, this form can be overly parameterised, so we consider a simpler special case where
\[
\Sigma_\bsx = \mathrm{diag}\big(\theta_1^2 ,\ldots,\theta_p^2 \big),
\]
i.e.\ the anisotropy is stationary while the nonstationarity is introduced only through an input-dependent amplitude.

Specifically, we define the following class of modified stationary kernels:
\[
S_\sigma
=
\left\{
\ker \in \mathcal{F}(\mathbb{R}^p \times \mathbb{R}^p,\mathbb{R}):
\ker(\bsx, \bsx') = \sigma(\bsx)\,\sigma(\bsx')\,
R^S\!\left(\|\bsx - \bsx'\|_{\theta}\right),
\ \sigma(\bsx)>0,\ \theta>0
\right\}.
\]
Here, nonstationarity arises solely from the spatially varying amplitude \(\sigma(\bsx)\), which modulates the 
local variance of the kernel.  
In practice, \(\sigma\) may be learned from the empirical variability of the data, or through a parametric model.

Another way to introduce nonstationarity into the kernel consists in replacing the constant length-scale parameters by input-dependent functions. 
This corresponds to a particular case of the nonstationary correlation function \(R_{\Sigma}^{NS}\) defined in \eqref{eq:NS_gau_ke}, obtained by choosing
\[
\Sigma_\bsx = \mathrm{diag}\bigl(\ell_1^2(\bsx),\ldots,\ell_p^2(\bsx)\bigr),
\]
where each \(\ell_s: \mathbb{R}^p \to (0,\infty)\) is an input-dependent length-scale function.

This leads to the following class of kernels:

\[
S_\ell =
\left\{
\ker \in \mathcal{F}(\mathbb{R}^p \times \mathbb{R}^p,\mathbb{R}):
\ker(\bsx, \bsx') = \sigma^2\, R^{NS}_{\Sigma}(\bsx, \bsx'),
\ \sigma>0
\right\},
\]

where the nonstationary correlation function \(R_{\Sigma}^{NS}\) takes the explicit form
\[
R_{\Sigma}^{NS}(\bsx, \bsx') =
\prod_{s=1}^p
\sqrt{
\frac{2\, \ell_s(\bsx)\, \ell_s(\bsx')}
     {\ell_s^2(\bsx) + \ell_s^2(\bsx')}
}
\, R^S\!\left(
\sqrt{
\sum_{s=1}^p
\frac{(x_s - x'_s)^2}
     {\tfrac{1}{2}\left(\ell_s^2(\bsx) + \ell_s^2(\bsx')\right)}
}
\right).
\]

In this construction, nonstationarity arises from spatially varying local smoothness,
while the global variance parameter \( \sigma^2 \) remains constant.
In the particular case where \( R^S(r) = \exp(-r^2/2) \), i.e.\ when the base correlation
is Gaussian, the above expression reduces to the well-known Gibbs kernel \cite{gibbs1998bayesian,williams2006gaussian}.

More generally, nonstationarity may be introduced simultaneously through both mechanisms.
That is, one may consider kernels in which both the local variance and the local length-scales
are input-dependent. This leads to the following general class:
\[
S_{\sigma,\ell}
=
\left\{
\ker \in \mathcal{F}(\mathbb{R}^p \times \mathbb{R}^p,\mathbb{R}):
\ker(\bsx, \bsx') = \sigma(\bsx)\,\sigma(\bsx')\, R^{NS}_{\Sigma}(\bsx, \bsx'),
\ \sigma(\bsx)>0
\right\}.
\]

In this formulation, nonstationarity arises both from spatially varying amplitude
and from spatially varying local smoothness.
While this class is highly expressive, it also introduces a larger number of parameters,
which may increase the risk of overfitting and the computational cost of inference.
In practice, a trade-off between flexibility and identifiability is therefore considered.

As for the distance function \( \dist(\cdot, \cdot) \) used to compare covariance matrices, several options are available depending on the application and the properties we wish to emphasize. One natural and computationally convenient choice is the \emph{spectral norm distance}, defined by:
\[
 \dist_{\text{spec}}(K_1, K_2)  = \sigma_{\max}(K_1 - K_2) = \sqrt{\lambda_{\max}((K_1 - K_2)^\top(K_1 - K_2))},
\]
where \( \lambda_{\max} \) denotes the largest eigenvalue and \( \sigma_{\max} \) is the largest singular value of the matrix difference. This distance reflects the worst-case deviation between the two matrices in terms of their spectral content.

An alternative is the \emph{Procrustes size-and-shape distance} \cite{pigoli2014distances}, which takes into account structural similarity between matrices. It is defined as:
\[
\dist_{\text{proc}}(K_1, K_2) = \inf_{R \in \mathcal{O}(\mathbb{R})} \left\| K_1^{1/2} - R K_2^{1/2} \right\|,
\]
where \( \mathcal{O}(\mathbb{R}) \) denotes the set of orthogonal matrices and $K_i^{1/2} $ is the positive square root of the positive matrix \(K_i\), for \(i \in \{1,2\}\), which can be defined from the spectral decomposition $K_i = U_i D_i U_i^\top$ as
$K_i^{1/2} = U_i D_i^{1/2} U_i^\top$, with \(U_i \in \mathcal{O}_n(\mathbb{R})\) and \(D_i\) a diagonal matrix containing the eigenvalues of \(K_i\).
This metric aligns the square roots of the covariance matrices under optimal rotation and thus compares their geometry more directly.

However, numerically manipulating these distances can be difficult, as one may encounter significant numerical errors, especially when using an empirical covariance matrix. Additionally, these distances are computationally heavy and costly. For this reason, in what follows, we will use the \emph{Frobenius distance}, which is both simpler and more efficient to compute. It is defined as:

$$
 \dist_{\text{Frob}}(K_1, K_2) = \|K_1 - K_2\|_F = \sqrt{\sum_{i,j} (K_1^{(i,j)} - K_2^{(i,j)})^2}  .
$$
In addition to its computational advantages, we observe that the Frobenius distance produces consistently good results for our specific use case, as will be presented in the following section.

\section{Application to linearized Burgers' equation}
\label{sec:linburgers}
\subsection{Implementation of the methodologies}

We now illustrate the proposed multifidelity methodologies by considering a simplified version of Burgers' equation. More precisely, we study the linearized model
\begin{align}
\partial_{t} u + \alpha(x) u_0(x) \partial_{x} u - \nu(x) \partial_{x}^{2} u &= 0, \quad \forall (t, x) \in  (0,1] \times (-1,1)  , \\
u(0, x) &= -\sin (\pi x) ,\\
u(t, -1) &= u(t, 1) = 0 ,
\label{eq:burgersLinear}
\end{align}
where the parameter \(\nu(x) = 0.02\) controls the shock, \(\alpha(x) = 1 \) and the background state is fixed to the initial condition \(u_0(x) = -\sin(\pi x)\).
The high-resolution simulator provides the true solution on a uniform grid $X_H$ of size \(10 \times 10\).

To construct the low-resolution simulator, we assume that the parameters \(\alpha\) and \(\nu\) are constant but uncertain. Specifically, they are drawn from the uniform distributions
\[
\alpha \sim \mathcal{U}(0.8, 1.1),
\qquad
\nu \sim \mathcal{U}(0.015, 0.03).
\]
In the following experiment, we set \(\Nobs = 1000\). For each of these simulations generated by the low-resolution solver, we solve the PDE corresponding to the sampled values of \(\alpha\) and \(\nu\) using a finite-difference method on a uniform grid $X_L$ of size \(10 \times 20\).
We, then, apply and compare the three methodologies as follows, using \(M = 1201\) collocation points, including \(M_{\Omega} = 1000\) points in the interior domain.

\paragraph{{\MFker}: }
The empirical covariance \(\ker_L\) is approximated by a smooth kernel \(\ker_{\mathrm{opt}}\), using as admissible family \(S\) the class of anisotropic Gaussian kernels.  
The multifidelity kernel \(\ker_H^{*} = \rho^2 \ker_{\mathrm{opt}} + \ker_d\) is then obtained, and the PDE is solved following \textbf{Algorithm~1}.

\begin{figure}[H]
\centering
\includegraphics{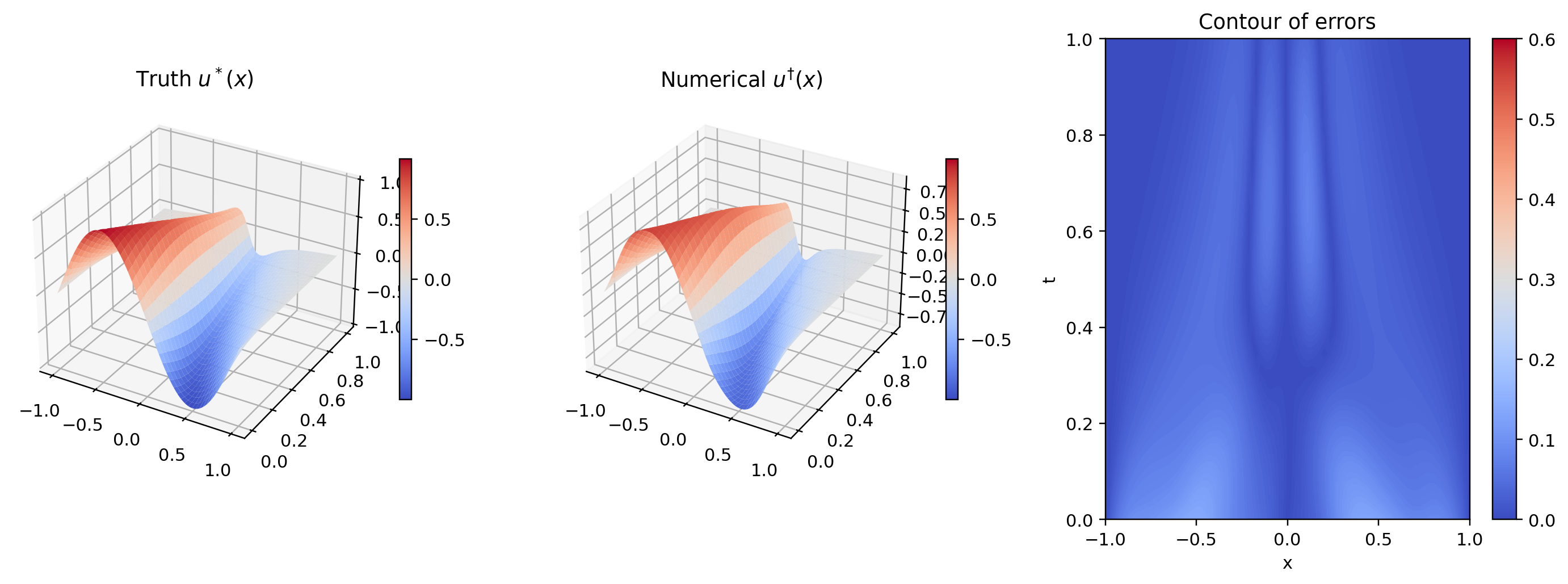}
\caption{{\MFker} result: $L_2$ error = 0.04 ; max error = 0.14 }
\label{fig:method1}
\end{figure}

The reconstructed solution in Figure~\ref{fig:method1}  captures the overall dynamics of the solution, although noticeable discrepancies remain in regions where the solution exhibits stronger gradients. This behavior is consistent with the error values reported in the caption.

\paragraph{{\MFmuker}:}
As in {\MFker}, the set \(S\) used to approximate \(\ker_L\) consists of anisotropic Gaussian kernels.  
Moreover, the high-fidelity mean \(\widehat{\mu}_H\) is computed using GPR with an anisotropic Gaussian kernel.  
Once \(\widehat{\mu}_H\) is constructed, the PDE is solved in the RKHS \(\mathcal{H}(\ker_H^{*})\) using shifted operators, following \textbf{Algorithm~2}.

\begin{figure}[H]
\centering
\includegraphics{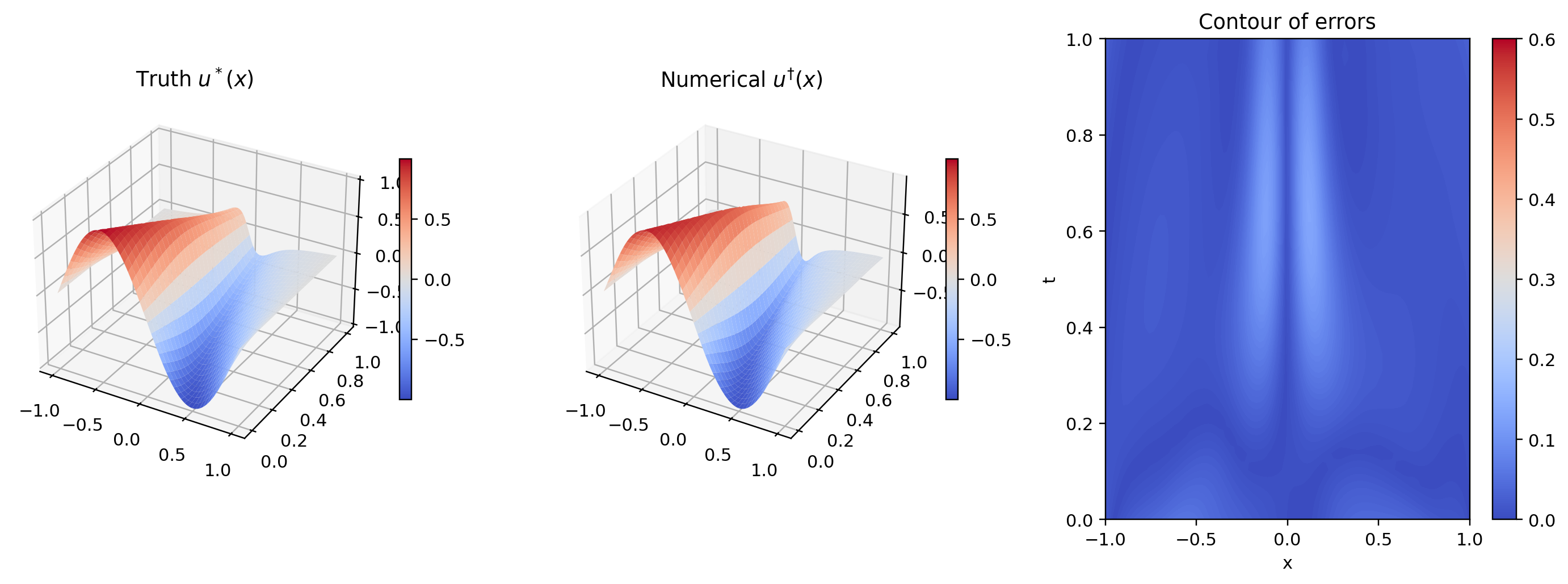}
\caption{{\MFmuker} result: $L_2$ error = 0.03 ; max error = 0.13 }
\label{fig:method2}
\end{figure}
The numerical solution in Figure~\ref{fig:method2}  is closer to the reference solution than in Figure~\ref{fig:method1}, and the error map indicates a slight reduction in the discrepancies. This highlights the benefit of including an explicit mean correction, in the case of the linearized Burgers' equation, in the RKHS formulation.

\paragraph{\MFmu:}
To avoid approximating \(\ker_L\), we learn the kernel directly from the high-fidelity data using a class \(S'\) of anisotropic Gaussian kernels.  
The high-fidelity mean \(\widehat{\mu}_H\) is again computed using an anisotropic Gaussian kernel.  
The PDE is then solved in the RKHS \(\mathcal{H}(\ker_H')\) following \textbf{Algorithm~3}.

\begin{figure}[H]
\centering
\includegraphics{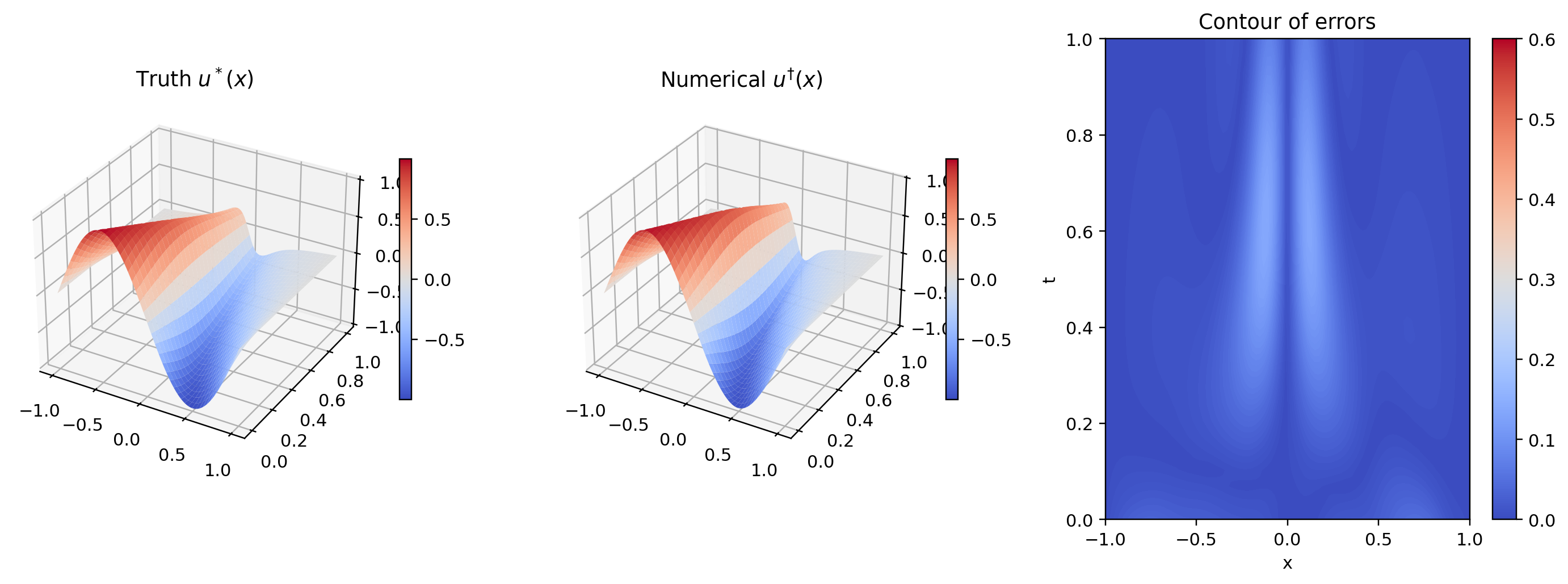}
\caption{{\MFmuker} result: $L_2$ error = 0.03 ; max error = 0.13}
\label{fig:method3}
\end{figure}

The reconstructed solution in Figure~\ref{fig:method3} closely matches the reference solution, with error levels comparable to those obtained with {\MFmuker}. This suggests that learning the kernel directly from high-fidelity residuals may provide a reliable alternative to multifidelity kernel construction.

\subsection{Discussion}
{\MFmuker} and {\MFmu} yield better results than {\MFker}, even though their optimisation problems 
are solved in the same RKHS.  
This behaviour can be demonstrated in a idealized theoretical setting.
In fact, when the mean of the high-fidelity Gaussian process coincides with the exact PDE solution, mean-corrected methodologies recover the solution exactly, whereas {\MFker} generally does not.\\

\begin{proposition}
Let $u$ denote the solution of the linearized Burgers' equation~\eqref{eq:burgersLinear}.
Let $\ker \in C^4(\mathbb{R}^2 \times \mathbb{R}^2)$ be a reproducing kernel, and let
$\{\bsx_m\}_{m=1}^M \subset [0,1]\times[-1,1]$ be a set of collocation points as defined above.
If the mean of the high-fidelity Gaussian process satisfies
\[
\mu_H = u ,
\]
then:
\begin{enumerate}
\item The solution $u^{\star\star}$ obtained using {\MFmuker} or {\MFmu} satisfies
\[
u^{\star\star} = u .
\]
\item Consequently,
\[
\|u^{\star\star} - u\| = 0 \le \|u^\star - u\|,
\]
where $u^\star$ denotes the solution obtained using {\MFker}.
\end{enumerate}
\end{proposition}

\begin{proof}
Let $L$ be the linear operator defined by
\[
L u(t,x) = u_t(t,x) + \alpha\, u_0(t,x)\, u_x(t,x) - \nu\, u_{xx}(t,x),
\]

Following \cite{Reproducing_property}, the action of the lifted operator $L^l$ on $\ker$ is
\[
L^l \ker((t,x), \cdot)
= \frac{\partial \ker}{\partial t_1}((t,x), \cdot)
+ \alpha\, u_0(t,x)\, \frac{\partial \ker}{\partial x_1}((t,x), \cdot)
- \nu\, \frac{\partial^2 \ker}{\partial x_1^2}((t,x), \cdot).
\]

By \cite{Reproducing_property}, the reproducing property holds for $L$ in $\mathcal{H}(\ker)$, and the linearized Burgers' equation 
\eqref{eq:burgersLinear} can be written in reproducing-kernel form:
\begin{align}
\langle u, L^l \ker((t,x), \cdot) \rangle &= 0,
    && \forall (t,x) \in (0,1] \times (-1,1), \\
\langle u, \ker((0,x), \cdot) \rangle &= -\sin(\pi x),
    && \forall x \in (-1,1), \\
\langle u, \ker((t,-1), \cdot) \rangle &= \langle u, \ker((t,1), \cdot) \rangle = 0,
    && \forall t \in (0,1].
\label{eq:burgersLinear_RP}
\end{align}

Let $\{\bsx_m\}_{m=1}^M$ be collocation points in $[0,1] \times [-1,1]$, with
\[
\bsx_1,\ldots,\bsx_{M_\Omega} \in (0,1] \times (-1,1),
\qquad
\bsx_{M_\Omega+1},\ldots,\bsx_M \in  
([0,1] \times \{-1,1\}) \cup (\{0\} \times [-1,1]).
\]

Then {\MFker} solves:
\begin{equation}
\underset{u \in \mathcal{H}(\ker)}{\operatorname{minimize}} \, \|u\|_{\mathcal{H}(\ker)} \quad
\text{s.t.} \quad
\begin{cases}
\langle u, L^l \ker(\bsx_m, \cdot) \rangle_{\mathcal{H}(\ker)} = 0, & m = 1,\ldots, M_\Omega,\\[2mm]
\langle u, \ker(\bsx_m, \cdot) \rangle_{\mathcal{H}(\ker)} = y_m, & m = M_\Omega+1,\ldots, M ,
\end{cases}
\label{eq:Burgers_coll}
\end{equation}
where
\[
y_m = 
\begin{cases}
-\sin(\pi x_m), & \bsx_m \in \{0\} \times [-1,1], \\[1mm]
0, & \bsx_m \in [0,1] \times \{-1,1\}.
\end{cases}
\]

By the representer theorem \cite{argyriou2009there}, the minimizer has the form
\begin{equation}
u^\star(\cdot)
= \sum_{m=1}^{M_\Omega} c_m\, L^l \ker(\bsx_m, \cdot)
  + \sum_{m=M_\Omega+1}^{M} d_m\, \ker(\bsx_m, \cdot),
\label{eq:repr_form}
\end{equation}
where the coefficient vector $\alpha = (c^\top,d^\top)^\top$ is given, provided that the block matrix
\[
K=
\begin{pmatrix}
K_{LL} & K_{LB} \\
K_{BL} & K_{BB}
\end{pmatrix}
\]
is invertible, by
\[
\alpha = K^{-1}
\begin{pmatrix}
0 \\[1mm] y
\end{pmatrix}.
\]
The blocks of $K$ are defined by
\begin{itemize}
    \item 
\(
(K_{LL})_{ij}
=
L^l L^r \ker(\bsx_i,\bsx_j) \) for \(
\qquad 1\le i,j\le M_\Omega,
\)
\item
\(
(K_{LB})_{ij}
=
L^l \ker(\bsx_i,\bsx_{M_\Omega+j}),
\)
for $1\le i\le M_\Omega$ and $1\le j\le M-M_\Omega$,
\item
\(
(K_{BL})_{ij}
=
L^l \ker(\bsx_{M_\Omega+i},\bsx_j),
\)
for $1\le i\le M-M_\Omega$ and $1\le j\le M_\Omega$, and
\item
\(
(K_{BB})_{ij}
=
\ker(\bsx_{M_\Omega+i},\bsx_{M_\Omega+j}),
\)
for $1\le i,j\le M-M_\Omega$.
\end{itemize}
This yields the compact expression
\begin{equation}
u^\star(t,x) = g(t,x)^\top \alpha , 
\end{equation}
with \[
g(t,x)
=
\Big(
\{\, L^l \ker(\bsx_m,(t,x)) \,\}_{m=1}^{M_\Omega},
\;
\{\, \ker(\bsx_m,(t,x)) \,\}_{m=M_\Omega+1}^{M}
\Big)^\top .
\]
Consider now the case of mean correction ({\MFmuker} and {\MFmu}). Since the PDE is linear, one can instead solve for a correction term $h^{\star \star}$ such that
\[
u^{\star\star} = \mu_H + h^{\star \star}.
\]

The function $h^{\star \star}$ is obtained as the solution of the following problem:
\begin{equation}
\underset{h \in \mathcal{H}(\ker)}{\operatorname{minimize}} \, \|h\|_{\mathcal{H}(\ker)} \quad
\text{s.t.} \quad
\begin{cases}
\langle h, L^l \ker(\bsx_m, \cdot) \rangle_{\mathcal{H}(\ker)} = - L\mu_H(\bsx_m), & m = 1,\ldots, M_\Omega,\\[2mm]
\langle h, \ker(\bsx_m, \cdot) \rangle_{\mathcal{H}(\ker)} = y_m - \mu_H(\bsx_m), & m = M_\Omega+1,\ldots, M.
\end{cases}
\label{eq:Burgers_coll_corrected}
\end{equation}

Since $\mu_H = u$ and $u$ satisfies both the PDE and the boundary conditions, we have
\[
L\mu_H(\bsx_m) = (Lu)(\bsx_m) = 0,\quad m \le M_\Omega,
\qquad
\mu_H(\bsx_m) = u(\bsx_m) = y_m,\quad m > M_\Omega.
\]

Therefore, the constraints reduce to
\[
\langle h, L^l \ker(\bsx_m, \cdot) \rangle = 0,
\qquad
\langle h, \ker(\bsx_m, \cdot) \rangle = 0.
\]

By uniqueness of the minimum-norm solution in $\mathcal{H}(\ker)$, it follows that
$h^{\star \star} = 0$.
Hence,
$u^{\star\star} = \mu_H = u$,
and
$\|u^{\star\star} - u\| = 0
  \le 
\|u^\star - u\|
$.
\end{proof}
This shows that {\MFmuker} and {\MFmu}, in an idealized setting, can yield solutions that are never worse than those obtained using {\MFker}.
In contrast, the $L_2$ error of {\MFker}, $\|u^\star - u\|$, has no reason to be equal to zero.
Unless the kernel $\ker$, the randomly selected collocation set, and the regularity assumptions happen to be perfectly aligned with the true solution, 
the optimisation problem~\eqref{eq:Burgers_coll} generally does not recover $u$ exactly.
Hence, $\|u^\star - u\|$ is nonzero in general.\\

\begin{remark}

The above result is not specific to the linearized Burgers' equation and extends to any linear partial differential equation that can be formulated in reproducing-kernel form within an RKHS.
\end{remark}

\begin{remark}
Observe from the proof that, in the case of mean correction 
(Methodologies \MFmu~ and \MFmuker), the reconstruction error is exactly zero whenever $\mu_H = u$, independently of the choice of the reproducing kernel $\ker$.
In this situation, the correction term vanishes identically, and the choice of kernel parameters has no influence on the final solution.
\end{remark}

This result corresponds to an ideal situation in which the high-fidelity mean $\mu_H$ coincides exactly with the true solution $u$, i.e.\ when a very accurate low-fidelity simulator is available. In this case, mean-corrected approaches (Methodologies \MFmu~ and \MFmuker) are naturally superior. However, this setting is not representative of practical applications. In more realistic scenarios, as illustrated in the following sections, {\MFker} often provides more robust and accurate performance.

\section{Application to Burgers' equation}
\label{sec:burgers}
\subsection{The difficulty of choosing an initial kernel in a single-fidelity context}
We solve the following Burgers' equation using the methodology presented in \cite{Owhadi}:
\begin{align}
\partial_{t} u + \alpha(x) u \partial_{x} u - \nu(x) \partial_{x}^{2} u &= 0, \quad \forall (x, t) \in (-1,1) \times (0,1] , \\
u( 0, x) &= -\sin (\pi x) , \\
u(t, -1) &= u( t, 1) = 0 .
\label{eq:burgers}
\end{align}
The parameter \(\nu(x) = 0.02\) controls the shock, and \(\alpha(x) = 1 \).
To solve this PDE, we solve Problem \eqref{relaxation1} using the methodology presented in \cite{Owhadi}. For the collocation points, we uniformly select 1000 points from the interior domain $\Omega = (0, 1]\times  (-1, 1) $, and 201 points from the boundary $ \partial \Omega =  [0,1] \times \{-1,1\}  \cup \{0\} \times (-1,1)  $.

In what follows, the numerical results will be compared to the true solution, which is obtained using the Cole–Hopf transformation \cite{hopf1950partial} as :

\[
u(t,x)
=
-\frac{
\displaystyle \int_{\mathbb{R}}
\sin\!\left(\pi\left(x-\sqrt{4\nu t}\,z\right)\right)
\exp\!\left(
-\frac{\cos\!\left(\pi\left(x-\sqrt{4\nu t}\,z\right)\right)}{2\pi\nu}
\right)
e^{-z^2}\,dz
}{
\displaystyle \int_{\mathbb{R}}
\exp\!\left(
-\frac{\cos\!\left(\pi\left(x-\sqrt{4\nu t}\,z\right)\right)}{2\pi\nu}
\right)
e^{-z^2}\,dz
}.
\]

First, following \cite{Owhadi}, we consider a kernel of the form:
\[
\ker\left((t,x), (t',x')\right) = \exp\left(-\frac{(t - t')^2}{2\theta_1^2} - \frac{(x - x')^2}{2\theta_2^2}\right).
\]
Using the hyperparameters \( \theta_1 = 0.47\) and \( \theta_2 = 0.07 \),  as in \cite{Owhadi}, yields the reconstruction shown in Figure~\ref{fig:bestcase-error}.
\begin{figure}[H]
    \centering
    \includegraphics[width=1\linewidth]{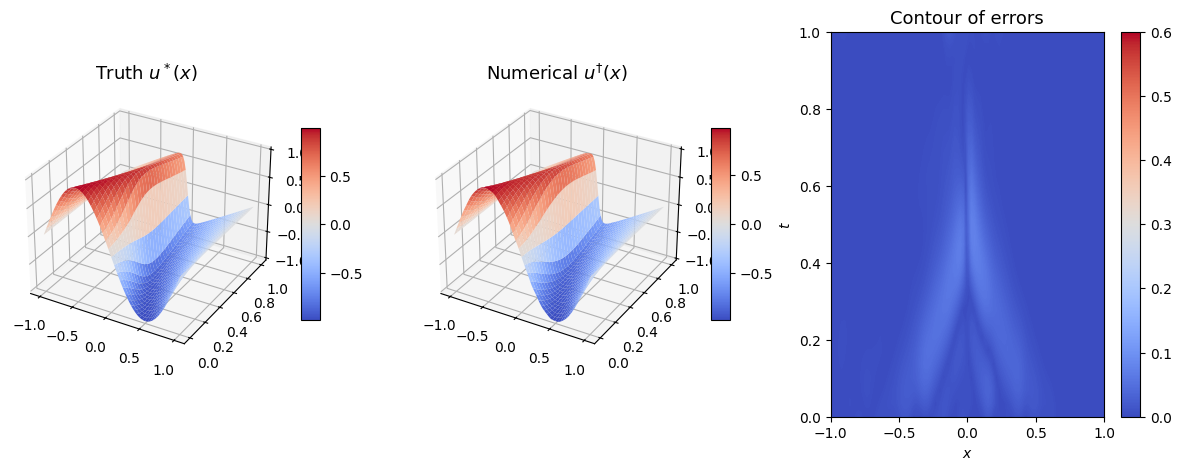}
    \caption{\( L_2 \) error = 0.01; max error = 0.09}
    \label{fig:bestcase-error}
\end{figure}

The numerical solution obtained in Figure~\ref{fig:bestcase-error} reproduces well the main structure of the Burgers dynamics, although the errors remain concentrated near the shock region. In this case, an appropriate manual choice of kernel parameters can lead to accurate solutions.

To assess the sensitivity to kernel parameter selection, we randomly sample the kernel hyperparameters \( \theta_1 \) and \( \theta_2 \) uniformly from the interval \( [0.01, 1] \times [0.01, 1] \). For each sampled parameter set, we apply the algorithm of \cite{Owhadi} and compute two error metrics over all simulations: the \(L_2\) error and the maximum (infinity-norm) error.

\begin{figure}
    \centering
    \includegraphics{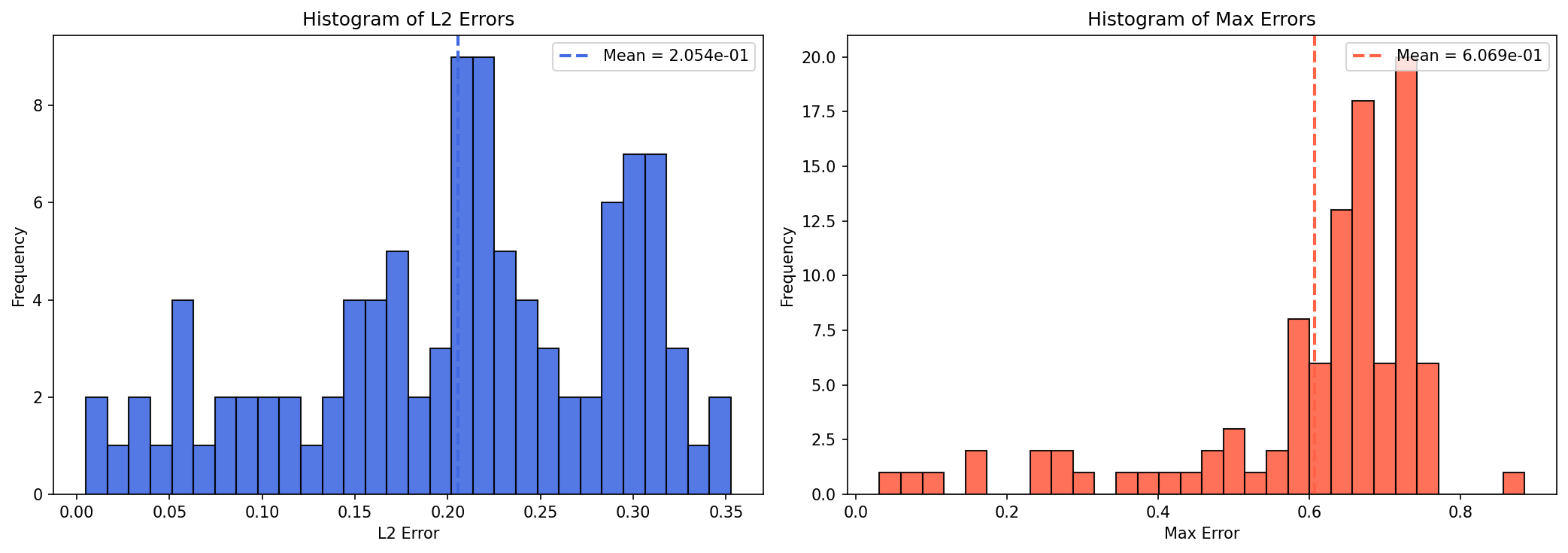}
    \caption{Histogram of errors of the random selection method}
    \label{fig:my_label}
\end{figure}

In Figure~\ref{fig:my_label}, we summarize the distribution of reconstruction errors obtained when the kernel parameters \( \theta_1 \) and \( \theta_2 \) are randomly sampled. The large variability observed in both the $L_2$ and maximum errors confirms the strong sensitivity of the method to kernel hyperparameters selection. This motivates the need for a kernel construction methodology. In the following subsection, we apply the multifidelity physics-informed framework for solving PDEs as presented in Section~\ref{Sec:PhI_methodology}.

\subsection{Improvements in a multifidelity context}

In this subsection, we return to the case of the Burgers' equation and consider a multifidelity setting in which two simulators are available: a high-resolution solver and a low-resolution one. The high-resolution simulator provides the reference solution on a uniform grid $X_H$ of size \(10 \times 10\).

To construct the low-resolution simulator, we assume that the parameters \(\alpha\) and \(\nu\) are constant but uncertain. Specifically, they are sampled from the uniform distributions
\[
\alpha \sim \mathcal{U}(0.8, 1.1), 
\qquad 
\nu \sim \mathcal{U}(0.015, 0.03).
\]
For each of the \(\Nobs\) realizations generated by the low-resolution model, we solve Burgers' equation using Fast Fourier transform \cite{seydaouglu2016numerical}, which provides a low-fidelity solution on a uniform grid $X_L$ of size \(10 \times 20\).

In all numerical experiments, we set $\Nobs = 1000$.
Among the different kernel classes introduced for the construction of the set $\mathcal{S}$, several choices are possible for solving the Burgers' equation in the multifidelity setting. 
In this subsection, we focus on the configuration that provides the best performance in our experiments, namely the \emph{non-stationary Gibbs kernel} combined with {\MFker}.
We consider the parametric form
\[
\ell_1(t,x) = \alpha_1 + \beta_1 t,
\qquad
\ell_2(t,x) = \alpha_2 + \beta_2 x,
\]
where the parameters
\(
\sigma^2, \alpha_1, \beta_1, \alpha_2, \beta_2
\)
are estimated from the data under the constraints
\(
\alpha_i > 0
\)
and
\(
\alpha_i + \beta_i > 0
\),
for \(i=1,2\).  

\begin{figure}[H]
\centering
\includegraphics{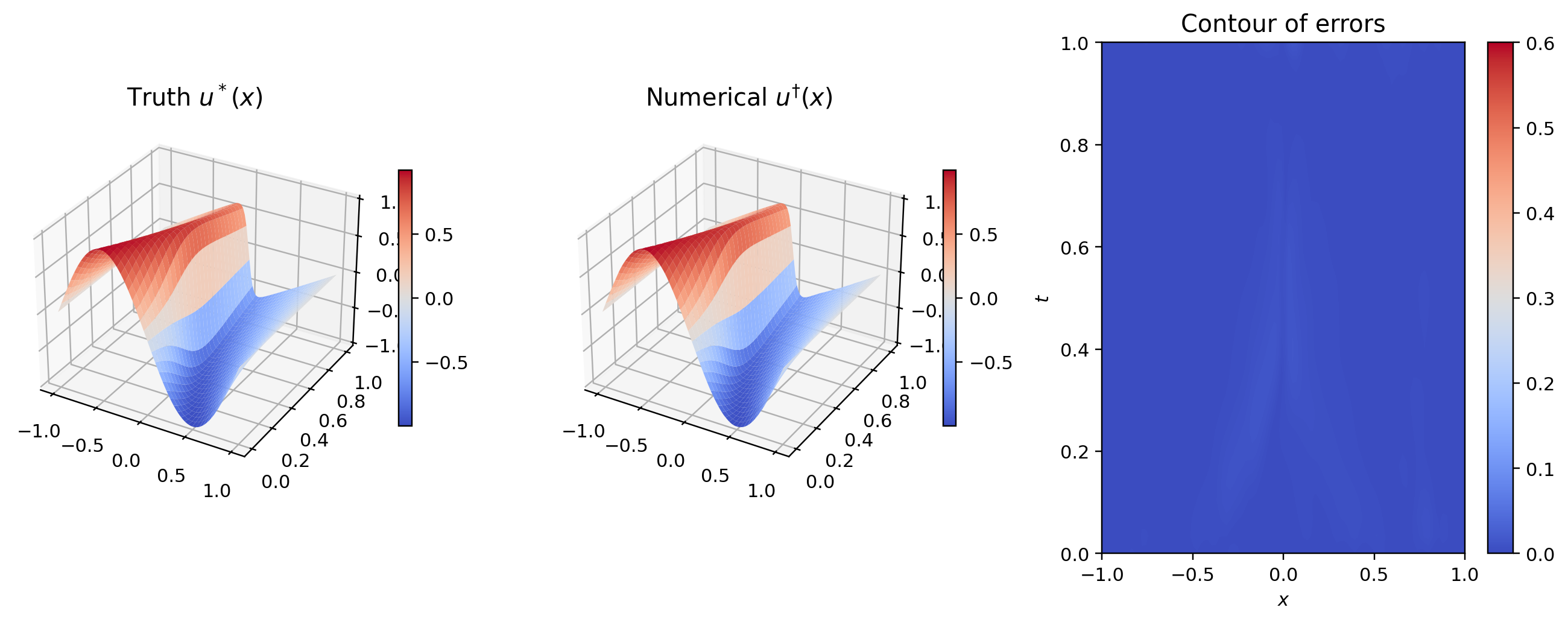}
\caption{{\MFker}  with non-stationary Gibbs kernel: 
$L_2$ error = 0.003; max error = 0.02.}
\label{fig:gibbs_method1}
\end{figure}

The result is presented in Figure~\ref{fig:gibbs_method1}. The numerical solution provides a satisfactory reconstruction of the reference field. The error map shows a clear reduction of the discrepancies compared with the previous kernel configurations. This suggests that introducing spatially varying length-scales, together with high-fidelity data regression, helps the kernel better adapt to the local dynamics of the solution.

Since the collocation points are randomly sampled from the interior and boundary domains, 
the reconstruction error may vary across realizations. 
To quantify this variability, we performed 80 independent realizations of the complete procedure.

The resulting errors can be summarized as
\[
L_2 = (4.20 \pm 0.77)\times 10^{-3},
\qquad
\|\cdot\|_{\infty} = (3.09 \pm 0.97)\times 10^{-2} .
\]
These results indicate that the proposed multifidelity approach is statistically stable, 
with limited variability induced by the random selection of collocation points.
This configuration yields the best accuracy among all tested models. 
The significant reduction in both the $L_2$ and maximum errors highlights the benefit of introducing spatially varying lengthscales, which allow the kernel to better adapt to the local dynamics of the solution.

\medskip

For the same non-stationary Gibbs kernel, we also applied Methodologies~\MFmu~ and~\MFmuker.  
In addition, we tested other kernel classes within $\mathcal{S}$, including:
\begin{itemize}
    \item stationary anisotropic Gaussian kernels,
    \item non-stationary Gaussian kernels with spatially varying variance.
\end{itemize}

For each of these configurations, the three methodologies were evaluated.  
The detailed numerical results are presented in Table~\ref{tab:burgers_kernel_comparison} and illustrated in Appendix~\ref{appendix:burgers_additional_results}.

\begin{table}[H]
\centering
\small
\begin{tabular}{|l||cc|cc|cc|}
\hline
\multirow{2}{*}{Kernel class $\mathcal{S}$} 
 & \multicolumn{2}{c|}{{\MFker}} 
 & \multicolumn{2}{c|}{{\MFmuker}} 
 & \multicolumn{2}{c|}{\MFmu} \\ \cline{2-7}
 & $L_2$  & Max  
 & $L_2$  & Max 
 & $L_2$  & Max \\ \hline \hline

Gibbs 
& $\underset{(0.77)}{\mathbf{4.20}}$ & $\underset{(0.97)}{\mathbf{3.09}}$
& $\underset{(0.77)}{\mathbf{4.19}}$ & $\underset{(0.95)}{\mathbf{3.08}}$
& $\underset{(0.081)}{57.7}$ & $\underset{(0.21)}{23.2}$ \\ \hline

Non-stationary Gaussian 
& $\underset{(0.70)}{4.53}$ & $\underset{(0.75)}{3.34}$
& $\underset{(0.71)}{4.52}$ & $\underset{(0.73)}{3.3}$
& $\underset{(0.050)}{\mathbf{17.7}}$ & $\underset{(0.17)}{\mathbf{9.9}}$ \\ \hline

Gaussian 
& $\underset{(0.68)}{4.49}$ & $\underset{(0.71)}{3.29}$
& $\underset{(0.68)}{4.48}$ & $\underset{(0.68)}{3.26}$
& $\underset{(0.034)}{59.8}$ & $\underset{(0.15)}{29.8}$ \\ \hline

\hline
\end{tabular}
\caption{
Comparison of the methodologies $\MFker$, $\MFmuker$, and $\MFmu$ across different kernel classes in $\mathcal{S}$ (stationary Gaussian, non-stationary Gaussian, and Gibbs). 
Results are reported as mean errors with standard deviations (in parentheses) over 80 runs. $L_2$ errors are scaled by $10^{-3}$ and Max errors by $10^{-2}$.
}
\label{tab:burgers_kernel_comparison}
\end{table}

Overall, the results show that {\MFker} and {\MFmuker} lead to very similar levels of accuracy when the kernel is learned from multifidelity information. In contrast, {\MFmu}, which uses multifidelity information only through the mean, is less robust and generally produces larger errors. This suggests that, in the nonlinear Burgers setting, the construction of the covariance kernel is the most influential step for obtaining an accurate solution. The comparison between kernel classes also shows that introducing nonstationarity improves the reconstruction, especially when local variations of the length-scales are taken into account, as in the Gibbs kernel.

\subsection{Sensitivity analysis with respect to the constraints}

In this subsection, we analyze the sensitivity of the solution with respect to the imposed constraints. 
Since the behavior is similar for the three methodologies, we restrict this study to {\MFker} with the Gaussian kernel.

Recall that {\MFker} consists in solving the constrained minimization problem
\eqref{eq:Methodology1}, which includes two types of constraints: 
(i) PDE constraints and boundary conditions, and 
(ii) interpolation constraints from high-fidelity data.
Within this framework, we obtained a $L^2$ error of $(4.49 \pm 0.68 ) \times 10^{-3}$ 
 and a maximum error of $(3.29 \pm 0.71) \times 10^{-2} $ (see Appendix~\ref{app:BurgersStatGaussian}, Figure~\ref{fig:gaussian_method1_appendix}). We now remove one of these two constraints and compare the accuracy with that reference.

\bigskip
\noindent
\textbf{Case 1: Removing the interpolation constraints.}

We first remove the interpolation constraints and solve only the physics-constrained problem:
\begin{equation}
\underset{u \in \HS(\ker)}{\operatorname{minimize}} \; \|u\|_{\HS(\ker)}
\quad \text{s.t.} \quad
\begin{cases}
\mathcal{P}(u)(\bsx_{m}) = f(\bsx_{m}), & m = 1,\ldots, M_{\Omega},\\[2mm]
\mathcal{B}(u)(\bsx_{m}) = g(\bsx_{m}), & m = M_{\Omega}+1,\ldots, M.
\end{cases}
\label{eq:Methodology1_no_interp}
\end{equation}

The corresponding reconstruction is shown in Figure~\ref{fig:sensitivity_no_interp}.

\begin{figure}[H]
\centering
\includegraphics{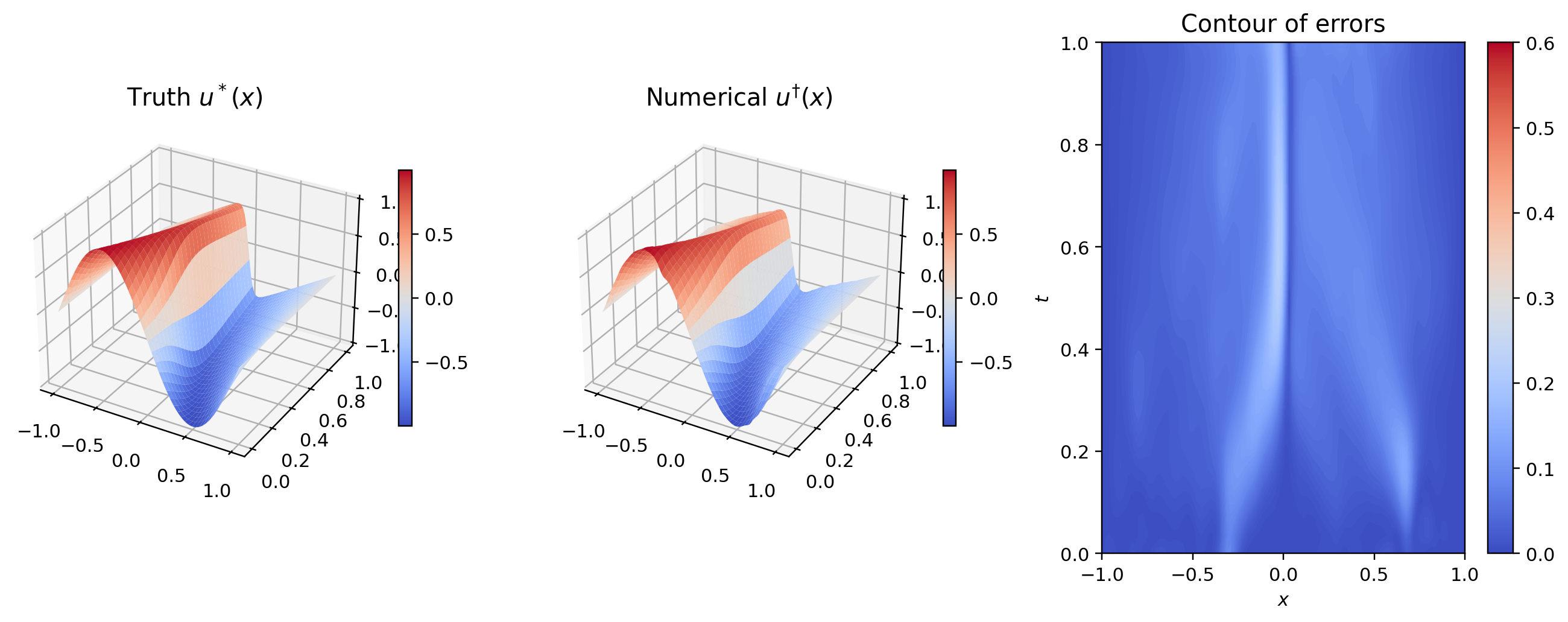}
\caption{Solution obtained without interpolation constraints. 
$L_2$ error = 0.06 ; max error = 0.22 .}
\label{fig:sensitivity_no_interp}
\end{figure}

Although the global structure of the solution in Figure~\ref{fig:sensitivity_no_interp} is recovered, the reconstruction exhibits larger deviations from the reference solution, indicating the benefit of including high-fidelity observations.

Over the 80 realizations performed with different random samplings of collocation points, 
the errors obtained in this configuration are
\[
L_2 = (3.09 \pm 2.00)\times 10^{-2},
\qquad
\|\cdot\|_{\infty} = (1.94 \pm 1.29)\times 10^{-1} .
\]

\bigskip
\noindent
\textbf{Case 2: Removing the PDE constraints.}

We now remove the PDE constraints and retain only the interpolation constraints, leading to the classical RKHS interpolation problem:
\begin{equation}
\underset{u \in \HS(\ker)}{\operatorname{minimize}} \; \|u\|_{\HS(\ker)} \quad
\text{s.t.} \quad 
u(\bsx_{H,i}) = y_{H,i}, \quad i = 1,\ldots, N_H.
\label{eq:Methodology1_no_pde}
\end{equation}

The corresponding reconstruction is shown in Figure~\ref{fig:sensitivity_no_pde}.

\begin{figure}[H]
\centering
\includegraphics{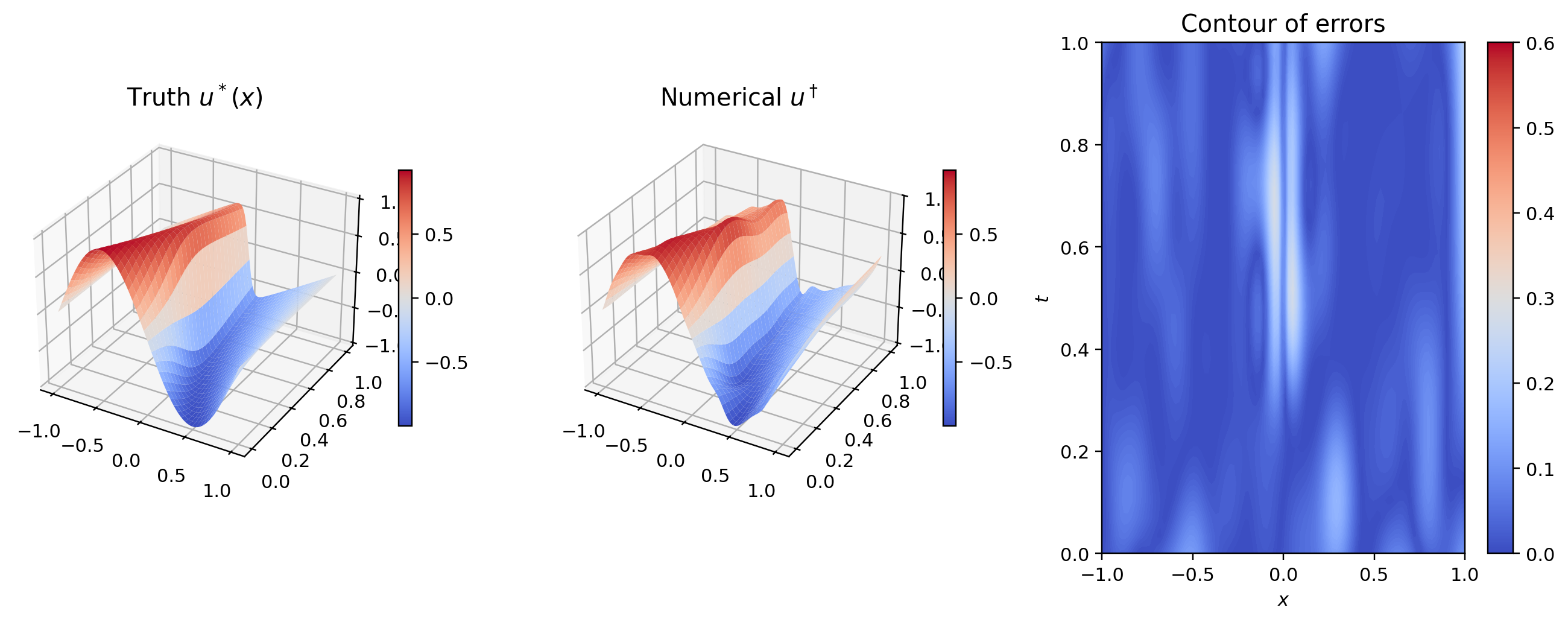}
\caption{Solution obtained with interpolation constraints only. 
$L_2$ error = 0.05 ; max error = 0.28 .}
\label{fig:sensitivity_no_pde}
\end{figure}
In this configuration, the solution interpolates the available high-fidelity observations but does not fully satisfy the PDE dynamics throughout the domain. As a consequence, larger errors appear away from the observation points.

Overall, the numerical experiments indicate that both interpolation constraints and PDE constraints provide meaningful information about the solution, even when enforced independently. Interpolation alone yields a reasonable approximation on the high-fidelity grid, while PDE constraints alone ensure physical consistency throughout the domain. However, the best accuracy and stability are obtained when both sources of information are enforced simultaneously. This confirms that combining deterministic high-fidelity data with PDE constraints enforced at randomly sampled collocation points within the proposed multifidelity physics-informed framework leads to a more reliable and accurate reconstruction, as each type of constraint compensates for the limitations of the other.

\subsection{Burgers' equation with spatially varying convection coefficient}
\label{subsec:burgers_alpha_var}

We now consider a configuration with spatially varying convection. 
We solve
\begin{align}
\partial_{t} u + \alpha(x)\, u\, \partial_{x}u - \nu\, \partial_{x}^{2}u &= 0,
\quad \forall (t,x)\in (0,1] \times (-1,1), \label{eq:burgers_alpha_var}\\
u(0,x) &= -\sin(\pi x), \nonumber\\
u(t,-1) &= u(t,1)=0 , \nonumber
\end{align}
with
\[
\alpha(x) = 1 + 0.2\,\sin(\pi x),
\qquad
\nu = 0.02.
\]
In this setting, the convection strength varies spatially, introducing additional non-stationarity in the dynamics and modifying the local propagation speed of the solution.

\paragraph{Low-fidelity simulator:}
The low-fidelity simulator is constructed as in Section~\ref{sec:burgers}. 
We assume spatially constant but uncertain parameters
\[
\alpha \sim \mathcal{U}(0.8,1.1),
\qquad
\nu \sim \mathcal{U}(0.01,0.03),
\]
centered around the reference value \(\nu=0.02\). 
For each realization, we solve \eqref{eq:burgers_alpha_var} using an FFT-based scheme \cite{seydaouglu2016numerical}, and use these simulations to learn the low-fidelity kernel.

\paragraph{Gibbs kernel results ({\MFker}):}
We apply {\MFker} using the non-stationary Gibbs kernel. The reconstruction is shown in Figure~\ref{fig:alpha_var_gibbs_m1}.

\begin{figure}[H]
\centering
\includegraphics{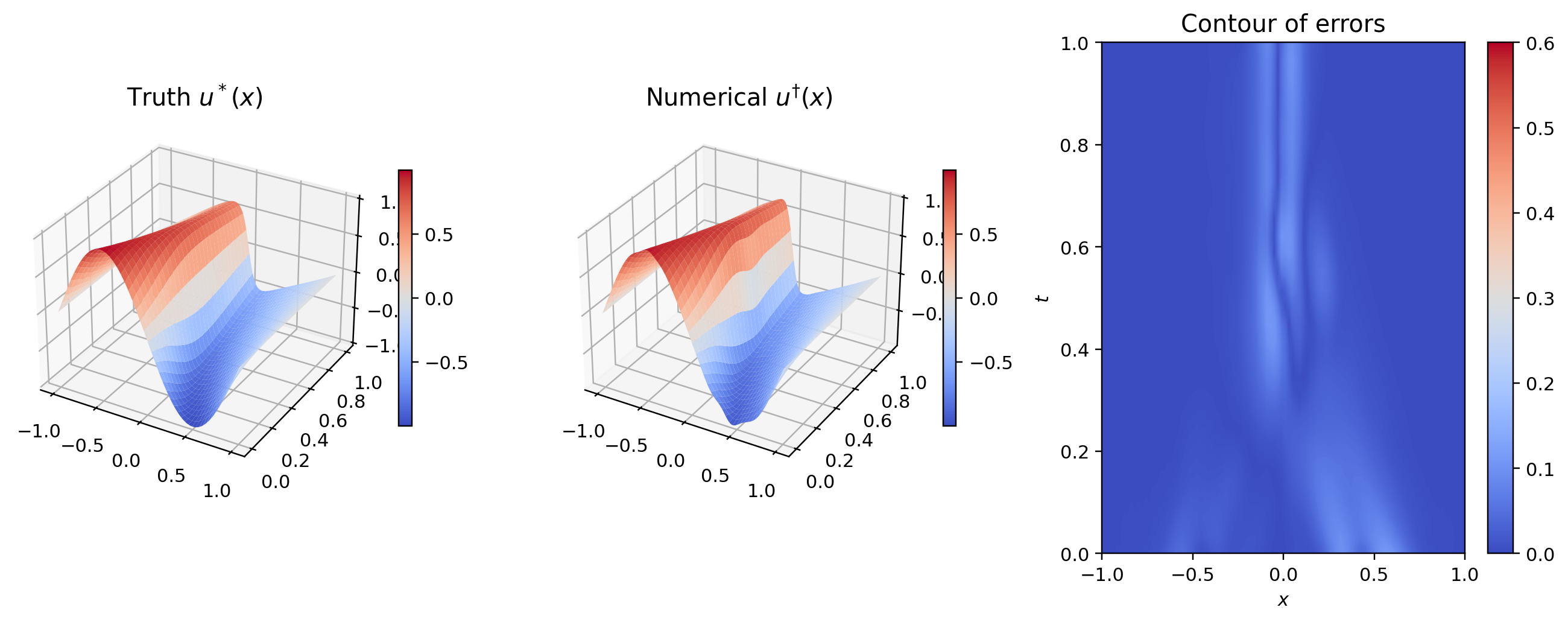}
\caption{Spatially varying \(\alpha(x)=1+0.2\sin(\pi x)\) -- {\MFker} (Gibbs kernel): $L_2$ error = 0.02; max error = 0.10.}
\label{fig:alpha_var_gibbs_m1}
\end{figure}

Despite the structural discrepancy between the low- and high-fidelity models, the multifidelity framework remains capable of capturing the main features of the solution. The error levels in Figure~\ref{fig:alpha_var_gibbs_m1} remain moderate, demonstrating the robustness of the proposed approach.

In all previously considered configurations, the low- and high-fidelity simulators belonged to the same model family: both solved Burgers' equation with constant coefficients, the discrepancy arising only from parameter uncertainty and numerical resolution. 
In contrast, the present configuration introduces a structural discrepancy between the two simulators. The high-fidelity solution corresponds to Burgers' equation with a spatially varying convection coefficient $\alpha(x)$, whereas the low-fidelity simulator still solves a constant-coefficient Burgers' equation with uncertain scalar parameters. 
Consequently, the mismatch between fidelities is no longer purely parametric but also structural. This setting therefore provides a more stringent test of the proposed multifidelity framework, as the cokriging model must compensate not only for numerical and parametric differences, but also for a genuine model-form discrepancy.
In this case, over 80 independent realizations, we obtain
\[
L_2 = (3.40 \pm 0.33)\times 10^{-2},
\qquad
\|\cdot\|_{\infty} = (2.11 \pm 0.24)\times 10^{-1}.
\]

\section{Conclusion and perspectives}
\label{sec:conclusion}
In this paper we have proposed a multifidelity physics-informed Gaussian process framework for solving nonlinear partial differential equations, where the solution is obtained as the Maximum A Posteriori estimator of a Gaussian process conditioned on both the physical constraints and the available data. A main idea of the approach is the construction of the prior in a multifidelity setting. Instead of fixing the covariance kernel in advance, we learn it from the available information. We first estimate an empirical covariance from low-fidelity simulations, regularize it into a smooth kernel. Then, we combine it with high-fidelity information within a multifidelity framework. Following the same philosophy, we may also construct the mean of the high-fidelity Gaussian process from the multifidelity information. This leads to three variants of the method: \textit{\MFker} incorporates multifidelity information through the kernel only; \textit{\MFmuker} combines multifidelity information through both the kernel and the mean; and \textit{\MFmu} uses multifidelity information through the mean only, while learning the kernel directly from high-fidelity residuals. In all cases, the resulting prior components are smooth and compatible with the PDE-constrained RKHS setting.

Our numerical results indicate that incorporating an explicit high-fidelity mean correction does not always lead to a significant improvement in the reconstruction. In practice, {\MFker} seems to provide the best compromise between accuracy, robustness, and simplicity. More generally, the experiments highlight that kernel learning plays an important role once strong physical constraints are enforced. They also confirm that combining PDE constraints with high-fidelity observations is essential: each source of information contributes complementary features, and their joint use leads to more accurate reconstruction than relying on either one alone.

A natural continuation of this work is the study of systems of PDEs, where several quantities are coupled. In that setting, extending the approach to multifidelity multi-output Gaussian processes would allow one to model both the interactions between variables and the different levels of fidelity in a unified way. This would make the framework applicable to more realistic problems, where complexity arises both from the physics and the structure of the available data.

\bibliographystyle{abbrv}
\bibliography{biblio}

\appendix

\section{Additional numerical results for the Burgers' equation}
\label{appendix:burgers_additional_results}

In this appendix, we report additional numerical results obtained for the Burgers' equation in the multifidelity setting.
As discussed in Section~\ref{sec:burgers}, several kernel classes were considered for the construction of the set $\mathcal{S}$, and each configuration was evaluated using Methodologies~\MFker, \MFmu~ and~\MFmuker.
The main text presented the best-performing configuration (non-stationary Gibbs kernel with {\MFker}).  
For completeness and reproducibility, we provide here the remaining results.

\subsection{Non-stationary Gibbs kernel}

We first report the results obtained with the non-stationary Gibbs kernel when combined with {\MFmuker} and {\MFmu}.

\begin{figure}[H]
\centering
\includegraphics{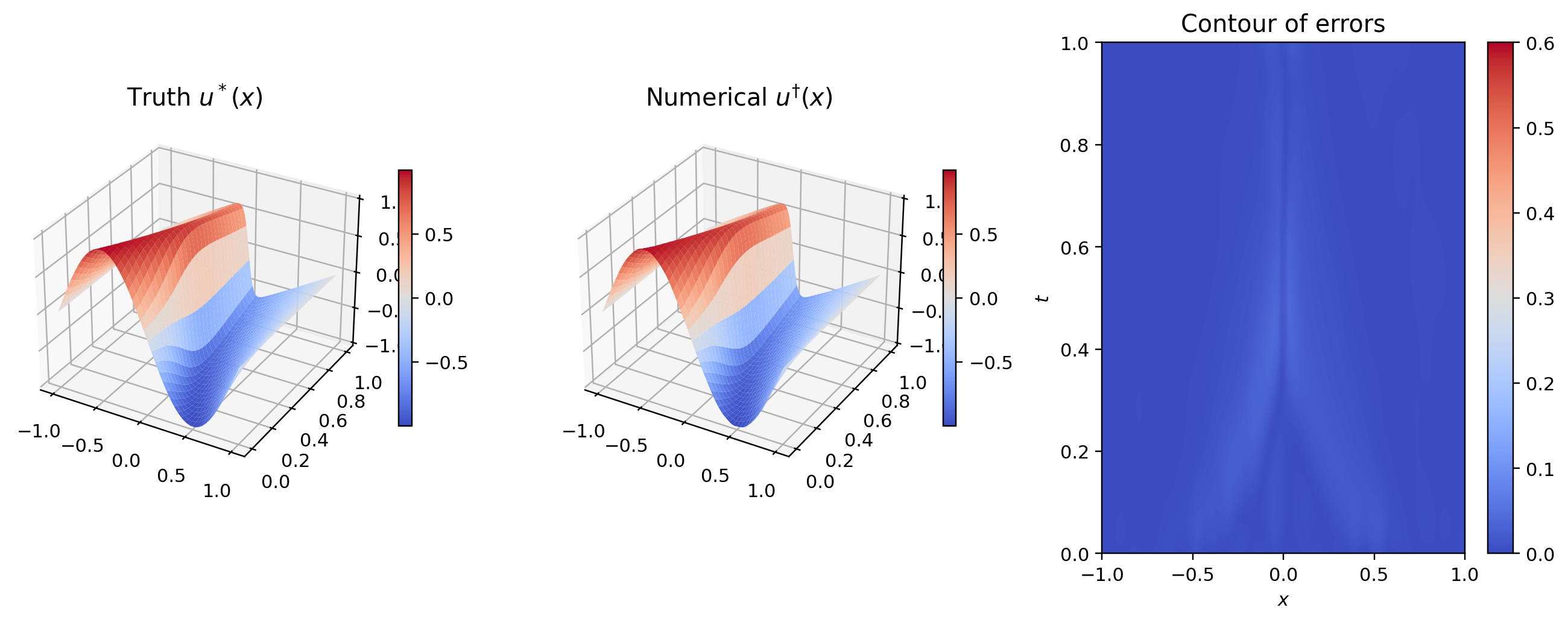}
\caption{Non-stationary Gibbs kernel -- {\MFmuker}: 
$L_2$ error = 0.008; max error = 0.03.}
\label{fig:gibbs_method2_appendix}
\end{figure}

\begin{figure}[H]
\centering
\includegraphics{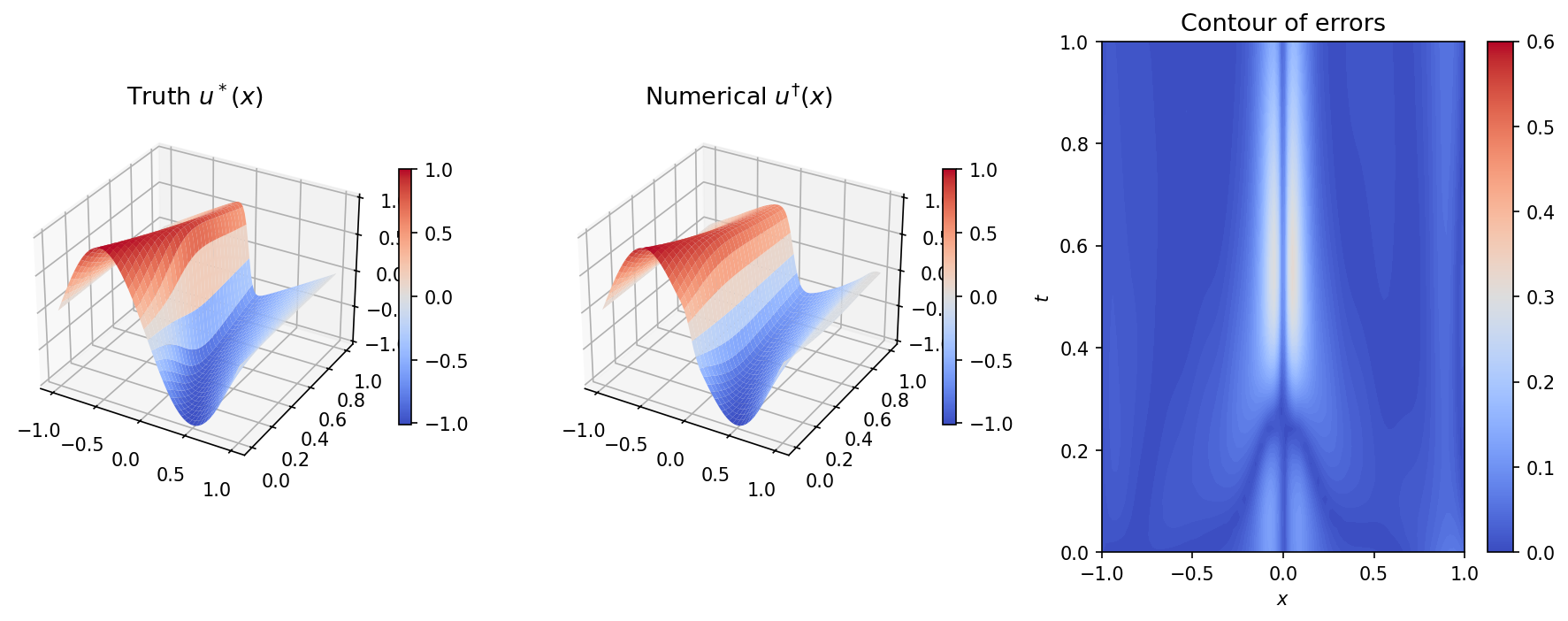}
\caption{Non-stationary Gibbs kernel -- {\MFmu} : 
$L_2$ error = 0.06; max error = 0.31.}
\label{fig:gibbs_method3_appendix}
\end{figure}

Over the 80 independent realizations, the errors obtained with {\MFmuker} are
\[
L_2 = (4.19 \pm 0.77)\times 10^{-3},
\qquad
\|\cdot\|_{\infty} = (3.08 \pm 0.95)\times 10^{-2}.
\]
For {\MFmu}, the corresponding errors are
\[
L_2 = (5.77 \pm 0.0081)\times 10^{-2},
\qquad
\|\cdot\|_{\infty} = (2.32 \pm 0.021)\times 10^{-1}.
\]

\subsection{Stationary anisotropic Gaussian kernel}
\label{app:BurgersStatGaussian}

We now consider the class of stationary anisotropic Gaussian kernels of the form

\[
\ker\left((t,x),(t',x')\right)
=
\exp\!\left(
-\frac{(t-t')^2}{2\theta_1^2}
-\frac{(x-x')^2}{2\theta_2^2}
\right),
\]

with hyperparameters estimated by maximum likelihood.

\begin{figure}[H]
\centering
\includegraphics{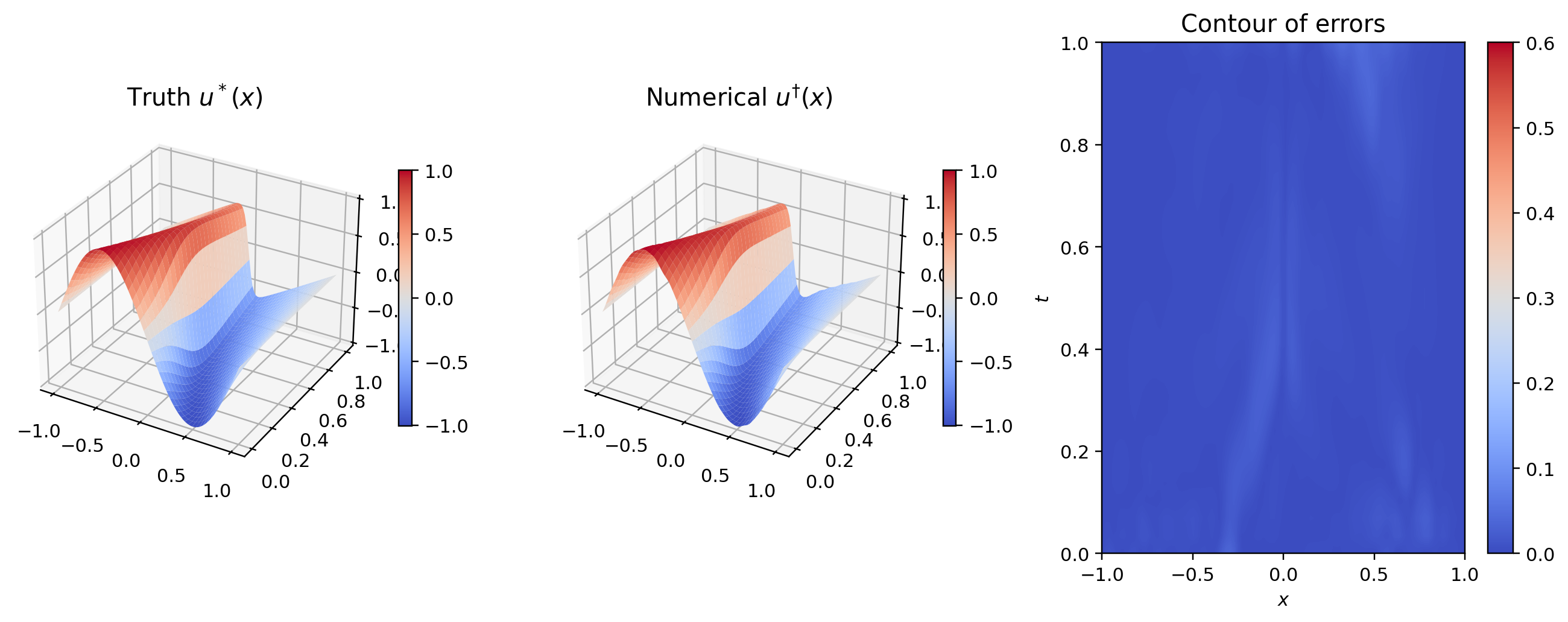}
\caption{Stationary Gaussian kernel -- {\MFker} ({\MFker}):
$L_2$ error = 0.008; max error = 0.04.}
\label{fig:gaussian_method1_appendix}
\end{figure}

\begin{figure}[H]
\centering
\includegraphics{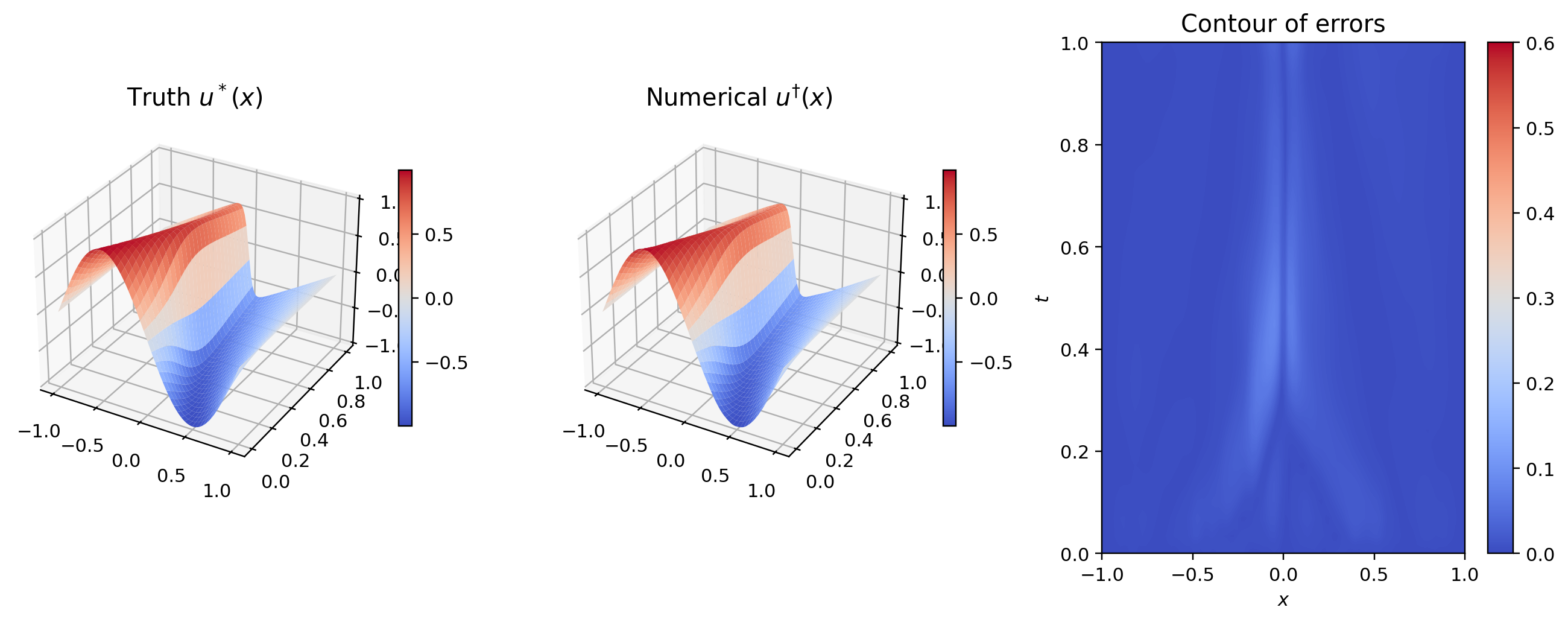}
\caption{Stationary Gaussian kernel -- {\MFmuker}: 
$L_2$ error = 0.01; max error = 0.08.}
\label{fig:gaussian_method2_appendix}
\end{figure}

\begin{figure}[H]
\centering
\includegraphics{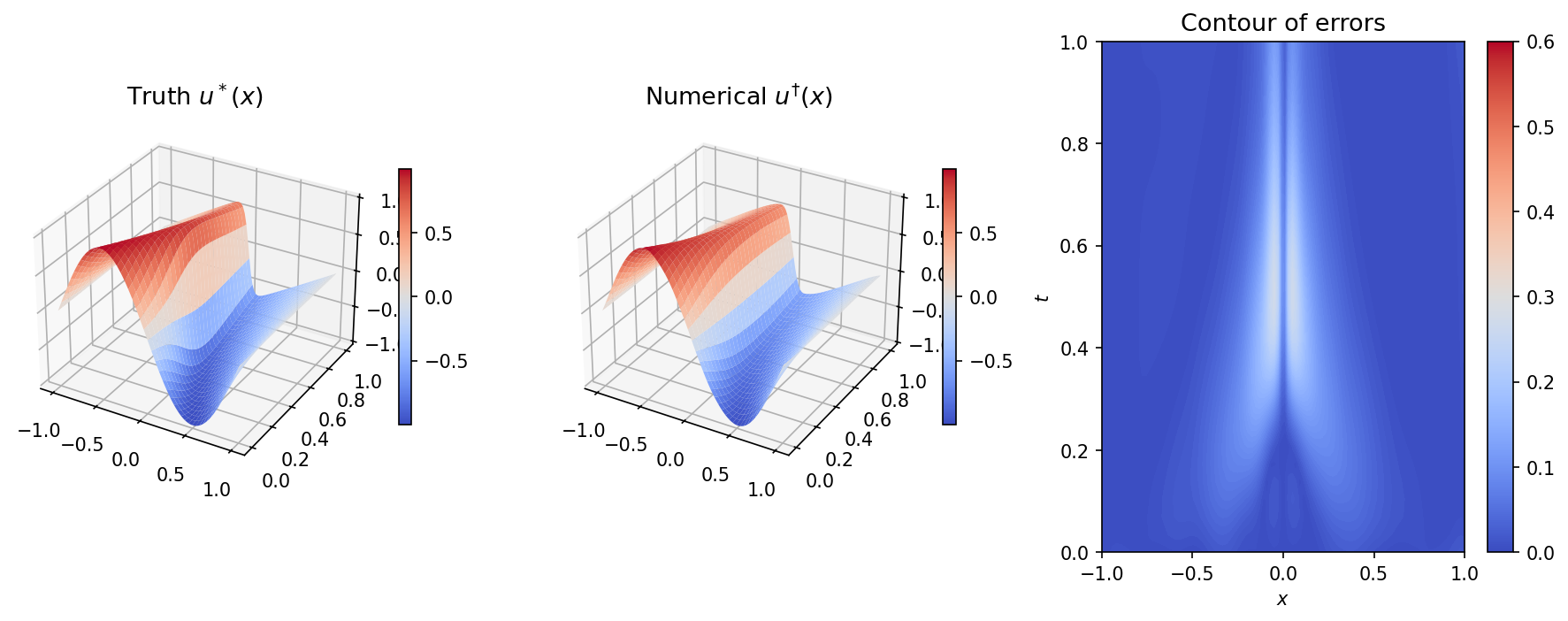}
\caption{Stationary Gaussian kernel -- {\MFmu}: 
$L_2$ error = 0.05; max error = 0.27.}
\label{fig:gaussian_method3_appendix}
\end{figure}

Over the 80 independent realizations, the errors obtained with {\MFker} are
\[
L_2 = (4.49 \pm 0.68)\times 10^{-3},
\qquad
\|\cdot\|_{\infty} = (3.29 \pm 0.71)\times 10^{-2}.
\]
For {\MFmuker}, the errors are
\[
L_2 = (4.48 \pm 0.68)\times 10^{-3},
\qquad
\|\cdot\|_{\infty} = (3.26 \pm 0.68)\times 10^{-2}.
\]
For {\MFmu}, the errors are
\[
L_2 = (5.98 \pm 0.0034)\times 10^{-2},
\qquad
\|\cdot\|_{\infty} = (2.98 \pm 0.015)\times 10^{-1}.
\]
\subsection{Non-stationary Gaussian kernel}

We now consider the non-stationary Gaussian kernel introduced in 
Section~\ref{subsec:nonstationary-kernel}. 
In this case, non-stationarity is introduced through a spatially varying 
standard deviation function $\sigma(x_1,x_2)$.

The choice of $\sigma$ is motivated by the empirical behavior of the variance 
observed from the low-fidelity simulations. 
For the Burgers' equation, the variance is symmetric with respect to $x=0$. 
For this reason, we restrict the visualization to $[0,1]\times[0,1]$.

\begin{figure}[H]
    \centering
    \includegraphics[width=1\linewidth]{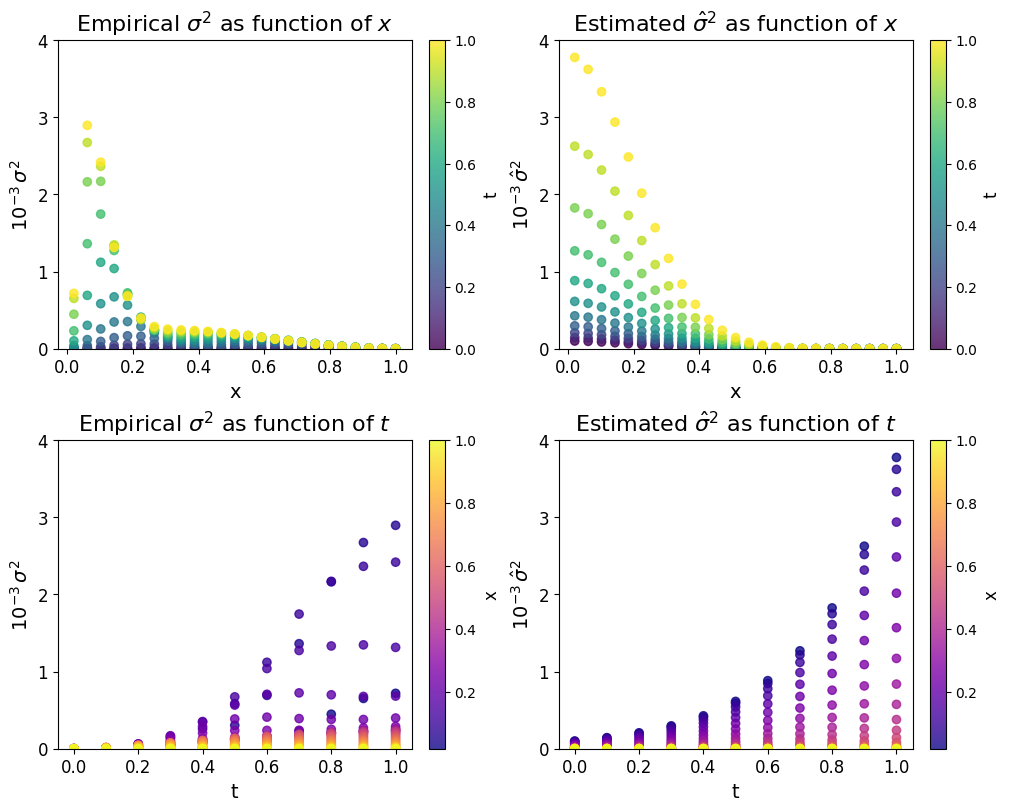}
    \caption{Comparison between empirical and estimated variances.}
    \label{fig:empirical_sigma_appendix}
\end{figure}

From these empirical observations, we note that the standard deviation 
approaches zero near $x=-1,0,1$ and $t=0$. 
To avoid instability in these regions, we introduce the following smooth 
parametric form:

\[
\sigma(x_1,x_2)
=
\exp\!\left(
\beta_0
+
\beta_1 x_2
+
\beta_2 x_2^2
+
\beta_3 x_1
\right),
\]
where the parameters $\beta_0,\beta_1,\beta_2,\beta_3$ are estimated by 
maximum likelihood.
We then apply the three methodologies using this non-stationary Gaussian kernel.

\begin{figure}[H]
\centering
\includegraphics{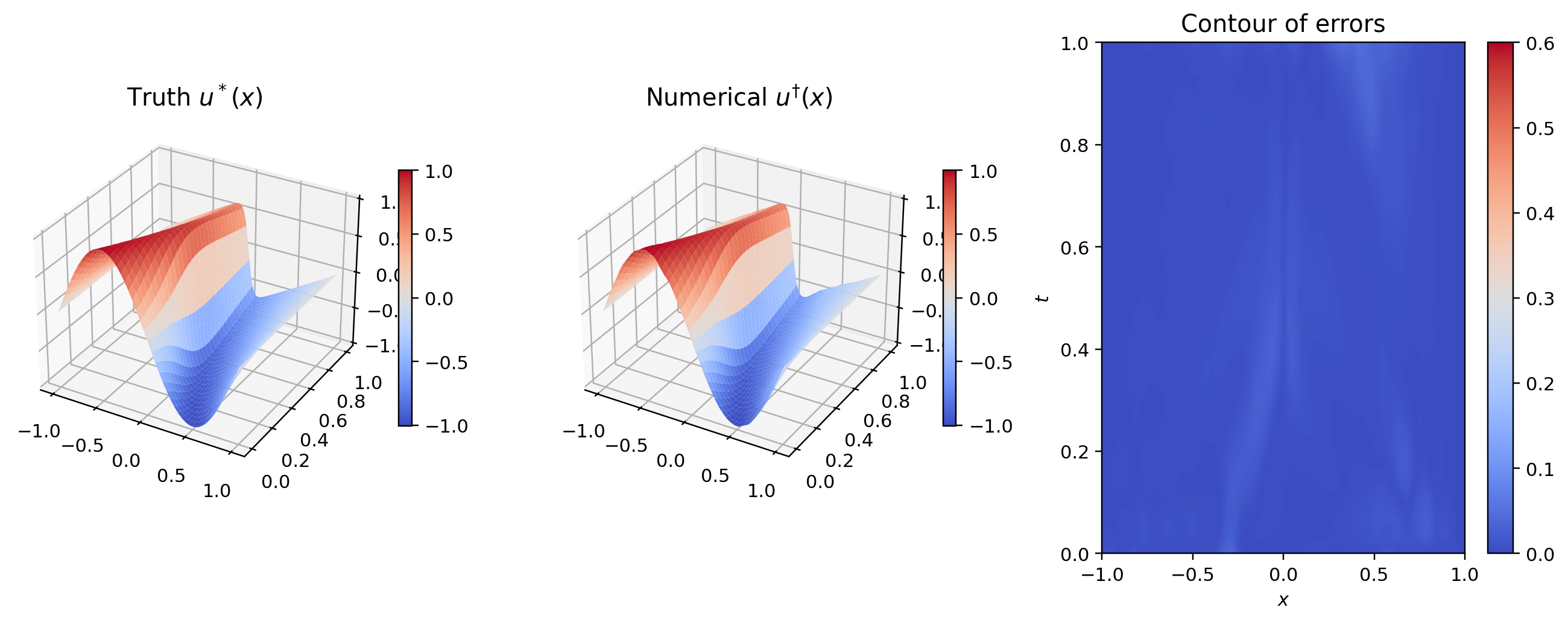}
\caption{Non-stationary Gaussian kernel -- {\MFker}:
$L_2$ error = 0.008; max error = 0.04.}
\label{fig:modified_gaussian_method1_appendix}
\end{figure}

\begin{figure}[H]
\centering
\includegraphics{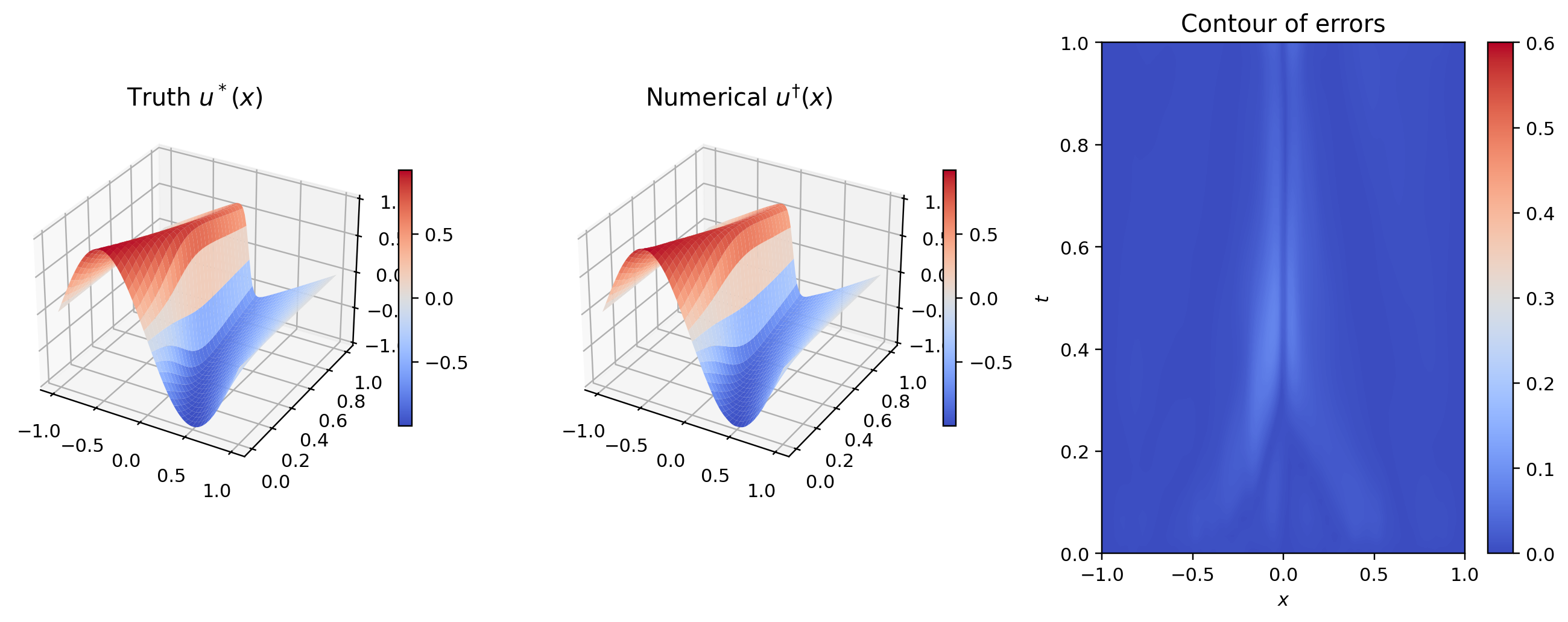}
\caption{Non-stationary Gaussian kernel -- {\MFmuker}: 
$L_2$ error = 0.01; max error = 0.08.}
\label{fig:modified_gaussian_method2_appendix}
\end{figure}

\begin{figure}[H]
\centering
\includegraphics{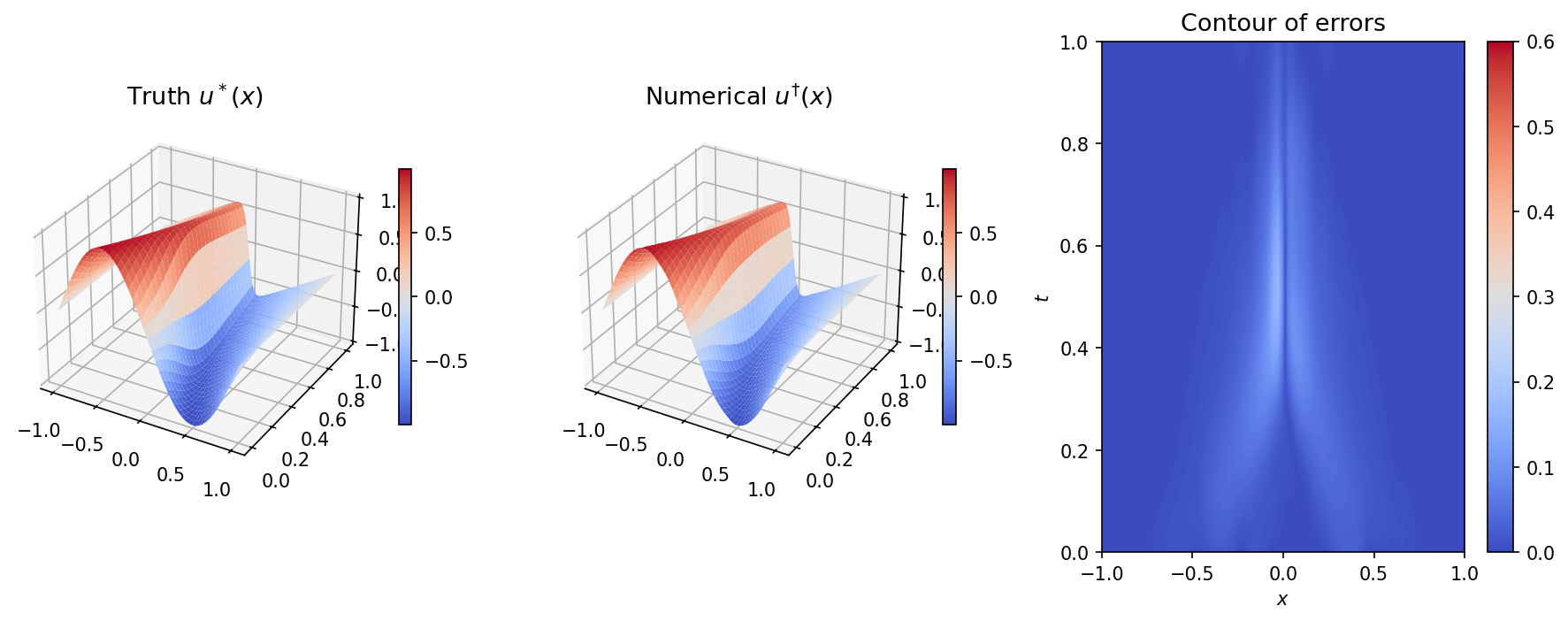}
\caption{Non-stationary Gaussian kernel -- {\MFmu}: 
$L_2$ error = 0.02; max error = 0.16.}
\label{fig:modified_gaussian_method3_appendix}
\end{figure}

Over the 80 independent realizations, the errors obtained with {\MFker} are
\[
L_2 = (4.53 \pm 0.70)\times 10^{-3},
\qquad
\|\cdot\|_{\infty} = (3.34 \pm 0.75)\times 10^{-2}.
\]
For {\MFmuker}, the errors are
\[
L_2 = (4.52 \pm 0.71)\times 10^{-3},
\qquad
\|\cdot\|_{\infty} = (3.31 \pm 0.73)\times 10^{-2}.
\]
For {\MFmu}, the errors are
\[
L_2 = (1.77 \pm 0.0050)\times 10^{-2},
\qquad
\|\cdot\|_{\infty} = (0.99 \pm 0.017)\times 10^{-1}.
\]

\subsection{Burgers with spatially varying convection coefficient}

\begin{figure}[H]
\centering
\includegraphics{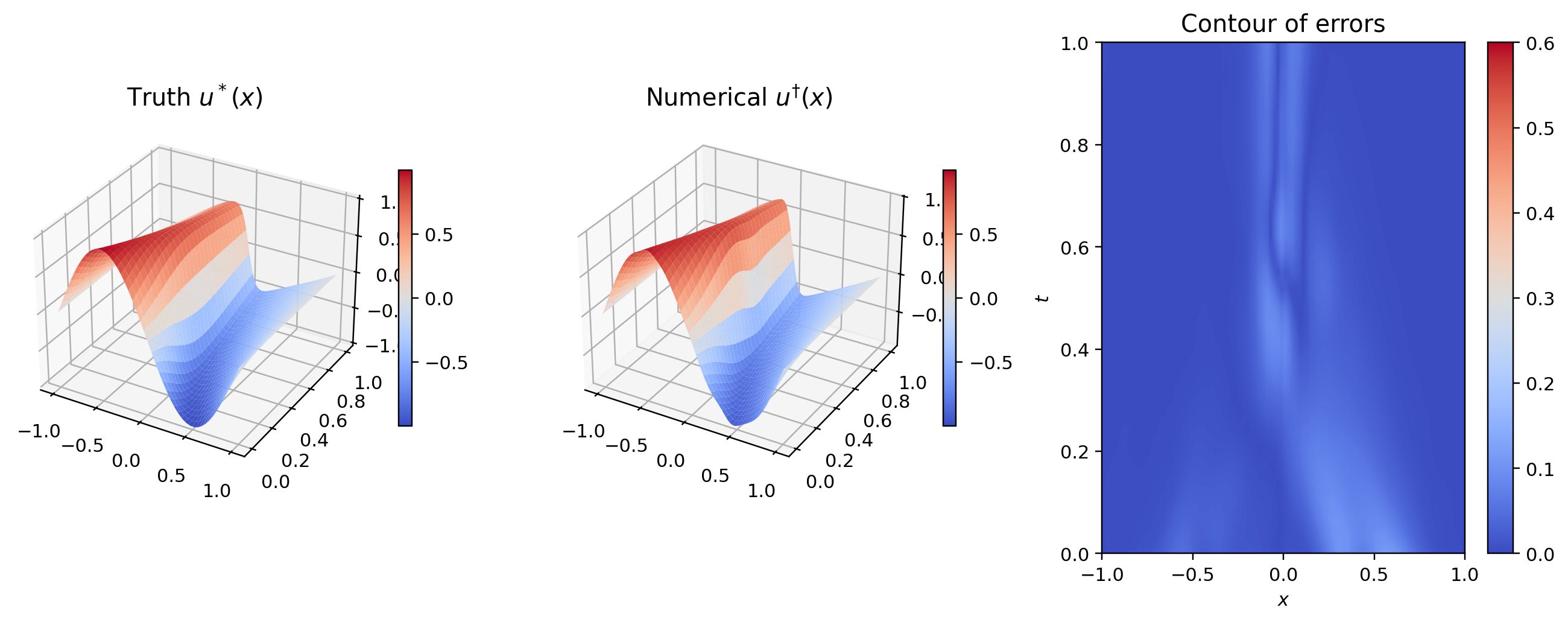}
\caption{Spatially varying \(\alpha(x)=1+0.2\sin(\pi x)\) -- {\MFmuker} (Gibbs kernel): $L_2$ error = 0.02; max error = 0.11.}
\label{fig:app_alpha_var_gibbs_m2}
\end{figure}

\begin{figure}[H]
\centering
\includegraphics{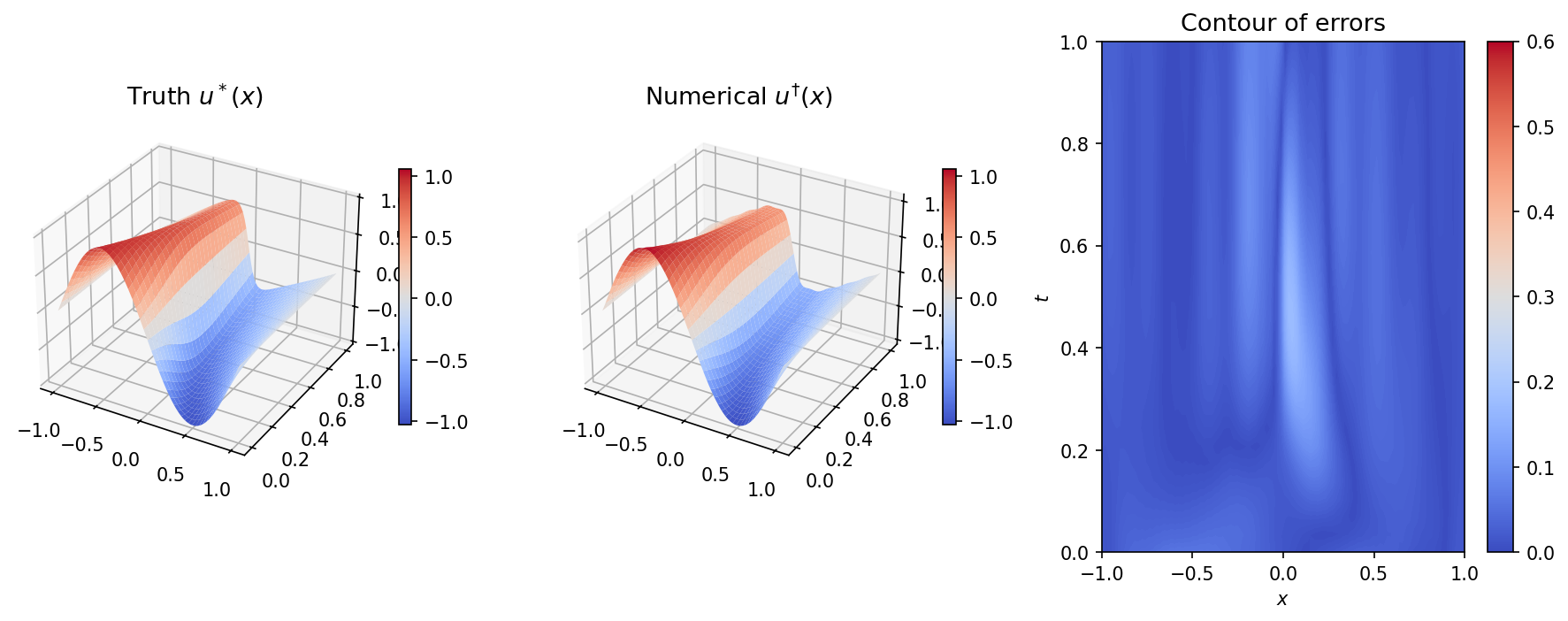}
\caption{Spatially varying \(\alpha(x)=1+0.2\sin(\pi x)\) -- {\MFmu} (Gibbs kernel): $L_2$ error = 0.03; max error = 0.17.}
\label{fig:app_alpha_var_gibbs_m3}
\end{figure}

As in the previous configurations, {\MFker} remains the most accurate approach, 
while {\MFmuker} and {\MFmu} lead to slightly larger errors. 
This confirms that, even in the presence of a structural discrepancy between 
low- and high-fidelity models, kernel learning from multifidelity information 
is the most critical step in constructing the solution within a kernel-based framework.

\end{document}